\documentclass[conference,letterpaper]{IEEEtran}
\IEEEoverridecommandlockouts
\usepackage{amsmath,amssymb,bm}
\usepackage{graphicx,booktabs,multirow,tabularx}
\usepackage{xcolor,xspace,amsthm,float}
\usepackage{algorithm,algpseudocode}
\usepackage[inline]{enumitem}
\usepackage{cite,url}
\usepackage[hidelinks]{hyperref}
\newtheorem{proposition}{Proposition}

\newcommand{\method}{{\upshape\scshape Spatial Grafting}\xspace}

\definecolor{DeltaPos}{RGB}{0,119,51}
\definecolor{DeltaNeg}{RGB}{187,34,34}
\newcommand{\dpos}[1]{\raisebox{0.5ex}{\tiny\textcolor{DeltaPos}{$+#1$}}}
\newcommand{\dneg}[1]{\raisebox{0.5ex}{\tiny\textcolor{DeltaNeg}{$-#1$}}}
\newlength{\gw}
\newcolumntype{B}{>{\raggedleft\arraybackslash}p{\gw}}
\newcolumntype{G}{>{\raggedright\arraybackslash}p{\gw}}
\newlength{\gwtmp}
\newcommand{\setgw}[2]{\settowidth{\gw}{#1}\settowidth{\gwtmp}{#2}%
  \ifdim\gwtmp>\gw\setlength{\gw}{\gwtmp}\fi}

\title{Spatial Grafting: Grounding 3D Features for Flow-Matching Robot Policies}
\author{\IEEEauthorblockN{Dingsheng Liu\IEEEauthorrefmark{1}\IEEEauthorrefmark{4}, Yangzheng Wu\IEEEauthorrefmark{2},
Mahboubeh Asadi\IEEEauthorrefmark{2}, Zhiyuan Li\IEEEauthorrefmark{1}\IEEEauthorrefmark{4}\\
Jinbang Huang\IEEEauthorrefmark{2}, Yixin Xiao\IEEEauthorrefmark{2},
Tongtong Cao\IEEEauthorrefmark{3}, Yingxue Zhang\IEEEauthorrefmark{2}}
\IEEEauthorblockA{\IEEEauthorrefmark{1}University of Toronto\\
\IEEEauthorrefmark{2}Huawei Noah's Ark Lab\\
\IEEEauthorrefmark{3}Department of Foundation Model, 2012 Labs}
\thanks{\IEEEauthorrefmark{4}Work done during the internship at Huawei Noah's Ark Lab.}}
\hypersetup{pdftitle={Spatial Grafting: Grounding 3D Features for Flow-Matching Robot Policies},pdfauthor={Dingsheng Liu, Yangzheng Wu, Mahboubeh Asadi, Zhiyuan Li, Jinbang Huang, Yixin Xiao, Tongtong Cao, Yingxue Zhang},pdfsubject={Robotic manipulation with metric-grounded spatial features}}

\begin{document}
\maketitle
\begin{abstract}
Pretrained robot manipulation policies such as vision-language-action models (VLAs) or world-action models (WAMs) leave interaction-relevant metric geometry implicit. Recent breakthroughs in spatial reconstruction can supply the necessary geometry reliably, but their features describe local shape without stating where it lies with respect to the robot. How best to deliver these features to a pretrained policy remains unresolved. We propose \textbf{\method}, a versatile, lightweight spatial module that binds frozen reconstruction features to metric, robot-relative geometry. \method constructs metric-grounded spatial tokens and injects them into the flow-matching action expert through cross-attention, without modifying the host's perceptual pathway, so the host retains the full benefit of its pretraining. We evaluate it more broadly than any geometry-aware
policy we compare against: one graft architecture, with no per-host redesign,
on two VLAs and two WAMs, across four simulation benchmarks that span short-horizon manipulation, visual robustness, clutter and long-horizon mobile manipulation, and on three real-robot platforms with single- and dual-arm configurations. On RoboTwin 2.0, a dual-arm manipulation benchmark, the graft improves every host across VLAs and WAMs. Grafted $\pi_{0.5}$ gains $11.3\%$ and $15.6\%$ on clean and randomized scenes, reaching $94.0\%$ and $92.4\%$, above the strongest published 3D-conditioned policy, WAM4D ($93.8\%$ and $89.9\%$). The margin widens as the horizon lengthens: on tasks from BEHAVIOR-1K, a dual-arm mobile manipulation challenge scored by average task progress, it surpasses the 2025 challenge winner on five of six tasks, by up to $0.47$ Q-score, and exceeds a map-conditioned spatial policy on average across the three tasks both report.
\end{abstract}

\section{Introduction}

Coupling a pretrained backbone with a flow-matching action expert has become
the dominant recipe for robot manipulation, with vision-language-action
models (VLAs) building the expert on a vision-language model (VLM) and
world-action models (WAMs) on a video world model
\cite{black2025pi0,yuan2026fastwam}. Neither backbone was designed for
manipulation or for the 3D structure it depends on: both read a scene as
pixels, with no notion of metric scale or of where a surface lies relative to
the gripper, so the policy selects the right object and the right sequence of
motions yet misplaces the grasp itself. What is missing is not coarse
behaviour but contact-level precision.
The problem this paper addresses is how to supply a flow-matching
policy with manipulation-relevant spatial information it can act on.

Spatial reconstruction foundation models such as Depth Anything~3 (DA3)
\cite{lin2026depthanything3} and VGGT-$\Omega$ \cite{wang2026vggtomega} are
the natural source. They recover dense metric structure from ordinary images,
and because they are learned networks rather than geometric pipelines, a
policy can read their latent features directly. A growing body of spatial
action models already adds 3D information to policies, but three gaps remain.
First, most are validated on one or two benchmarks and a single host, so it
is difficult to attribute a gain to the geometry rather than to the
surrounding design. Second, many depend on inputs a deployed robot does not
have: privileged instance labels or object state from a simulator, dedicated
depth or point-cloud sensors, or a scene map maintained across the episode.
Third, each is built for one policy family, often as a new 3D architecture
trained from scratch, and does not transfer to the strongest pretrained
hosts. The open question is therefore not whether geometry helps, but how to effectively
deliver it to flow-matching action policies.

\begin{figure*}[t]
    \centering
    \includegraphics[width=\textwidth]{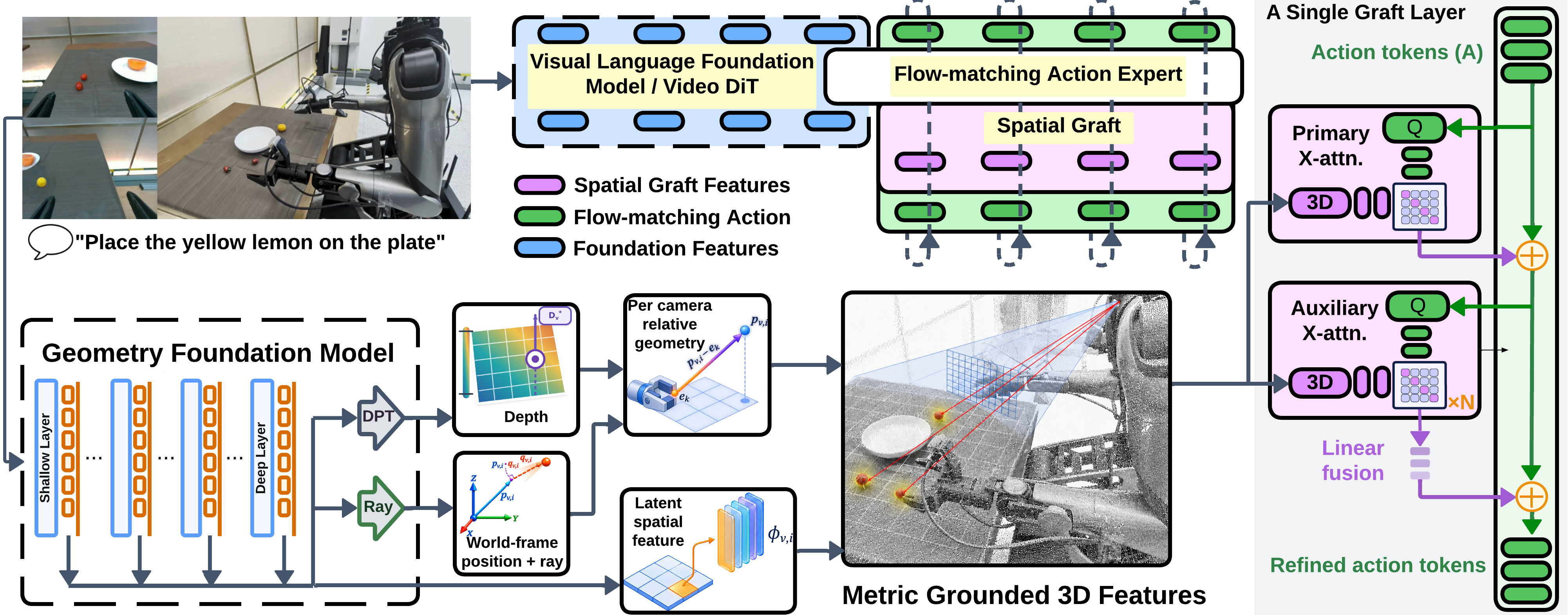}
    \caption{\textbf{Overview of \method.} A frozen geometry foundation model
    extracts multi-level features from each RGB view. Depth, calibrated rays,
    and end-effector anchors bind those features to metric, robot-relative 3D
    locations, forming a spatial bank per view. The magnified graft layer acts
    only on the flow-matching action expert: action tokens first cross-attend
    to the primary-view bank; the updated states then query the auxiliary-view
    banks. A single bias-free linear projection fuses the auxiliary attention
    outputs into the action update. This layer repeats in selected late action
    blocks, leaving the host's vision-language or video pathway ungrafted.}
    \label{fig:pipeline}
\end{figure*}

\method is that interface (Figure~\ref{fig:pipeline}). Rather than
converting reconstruction output into an explicit representation, we ground
the reconstruction model's latent space itself: each frozen latent is bound,
at its own grid location, to its metric position in the workspace and to its
offset from every end effector, and the bound tokens are grafted by
cross-attention into the final blocks of the host's action expert. The graft
needs RGB, camera calibration and a depth map, measured or predicted by the
same frozen model; it uses no segmentation, no object state, and keeps no map
or state between observations. The host's perception pathway is untouched,
the added parameters are 4--5\% of a 3--5B host, and the graft is routed so that
it cannot help without reading the geometry. Because the interface asks only
for a flow-matching action expert, one construction serves both VLAs and
WAMs. We evaluate it more broadly than any spatial action model we compare
against (Table~\ref{tab:spatial_vla_comparison}): four hosts, two
VLAs and two WAMs; four simulation benchmarks that each isolate a different
demand---short-horizon manipulation on LIBERO~\cite{liu2023libero}, visual
robustness on RoboTwin~2.0~\cite{chen2025robotwin2}, clutter and collision
on RoboPRO~\cite{li2026robopro}, and long-horizon mobile manipulation on
BEHAVIOR-1K (B1K)~\cite{behaviorchallenge2025}; and three real robots,
single-arm, bimanual and humanoid. Our main contributions are threefold:
\begin{itemize}
\item \textbf{Robot-grounded reconstruction latents.}
Prior work has used reconstruction latents or metric geometric grounding separately. We bind each latent to metric and end-effector-relative coordinates, turning a frozen reconstruction space into a robot-grounded representation for direct action prediction. Bank-off and bank-shuffle ablations show that the gain comes from grounded content rather than added capacity (Section~\ref{sec:ablations}).
\item \textbf{A grafting interface for any flow-matching action expert.}
The graft uses bank-only cross-attention in the final action-expert blocks, which are revisited at every denoising step, and requires no other host modification. With bias-free projections, zero banks induce zero change (Proposition~\ref{prop:bank-only}). The same interface improves two VLAs and two WAMs from public checkpoints and remains effective across reconstruction backbones (Sections~\ref{sec:results} and~\ref{sec:ablations}).
\item \textbf{Broad evaluation across benchmarks and embodiments.}
Across four hosts, four benchmarks, and three real robots, we use only RGB, calibration, and measured or predicted depth, including simulator-provided depth in RoboTwin. Grafted $\pi_{0.5}$ outperforms the strongest published 3D-conditioned policy on RoboTwin and improves the B1K challenge winner's performance on five of six tasks; across both RoboPRO hosts, grafting improves all eight host--condition entries (Sections~\ref{sec:results} and~\ref{sec:complex-settings}, Tables~\ref{tab:libero-robotwin}--\ref{tab:real-robot}).
\end{itemize}

\section{Related Work}
\label{sec:related}

\paragraph{Action models and reconstruction models}
Modern manipulation policies attach a continuous action expert, trained by
diffusion or flow matching \cite{chi2023diffusion,black2025pi0,liu2025rdt},
to a pretrained backbone: a VLM in VLAs
\cite{octo2024,black2025pi05,zheng2026xvla} or a video world model in WAMs
\cite{yuan2026fastwam,li2026lingbotva}. Both backbones perceive in 2D, so the
expert receives neither metric scale nor the position of a surface relative
to the gripper, and the residual failures concentrate at contact.
Reconstruction foundation models are the complementary resource: they regress
dense metric depth and pointmaps from ordinary images
\cite{ranftl2022midas,yang2024depthanything,yang2024depthanythingv2,wang2024dust3r,leroy2024mast3r,wang2026vggtomega,lin2026depthanything3},
and their latent features encode structure learned far beyond any robot
dataset.

\paragraph{Spatial action models}
Prior work has supplied geometry to policies as explicit signals: point
clouds through a branch into the action module
\cite{li2026pointvla,sun2025geovla}, estimated-depth position encodings on
the VLM's tokens \cite{qu2025spatialvla}, depth tokens
\cite{yuan2026depthvla,lee2026molmoact}, and persistent scene maps
\cite{kim2026serf}; earlier 3D policies attended to lifted features relative
to the gripper but trained from scratch \cite{gervet2023act3d,ke20243ddiffuser}.
These designs pay for their geometry with an RGB-D or point-cloud sensor, with
privileged simulator labels and a prebuilt map, or with a bespoke
architecture, and each serves one policy family. A second line uses
reconstruction \emph{latents} instead, as a training-time alignment
target \cite{li2026spatialforcing,guo2025glad,li2026wam4d,zhang2026mecowam},
so that no geometry is present at test time and the gain is bounded by what
the host internalizes. Alternatively, the latents are fused into the VLM's
visual tokens
\cite{yu2026threedmix,yang2026impactgfm}, where the latent arrives in the
camera's frame at an unknown scale and is compressed by the trunk before the
action expert sees it. Both lines report at most three evaluation settings
on a single host (Table~\ref{tab:spatial_vla_comparison}), and the
two studies that compare injection routes find that where geometry enters
decides the outcome \cite{yu2026threedmix,yang2026impactgfm}.

\paragraph{The missing interface}
Taken together, prior work has treated \emph{where geometry comes from}
(a sensor or a reconstruction model) and \emph{where it enters} (the input,
the loss or the action module) as separate questions, and has not delivered
reconstruction latents to a pretrained action expert in the frame a
controller needs. A latent describes the scene relatively, in the camera's
frame and up to scale; a policy must learn a correlation between what it
observes and where, in absolute metric units and relative to its end
effectors, it should act. \method supplies that interface. It binds each
frozen latent, at its own grid location, to its absolute metric position and
to its offset from every end effector, so no point cloud, map or privileged
label is needed; it delivers the bound tokens through a bank-only route into
the final blocks of the action expert, so added capacity cannot help without
the geometry; and it asks the host for nothing but a flow-matching expert, so
one construction serves both VLAs and WAMs. We compare against six of the methods
above: SERF on B1K, and SpatialVLA, GeoVLA, Spatial Forcing, GLaD and WAM4D
on LIBERO and RoboTwin.

\subsection{Detailed related-work comparison}
\label{app:related-comparison}

\paragraph{Geometry signals and host policies}
Table~\ref{tab:spatial_vla_comparison} expands the comparison in
Section~\ref{sec:related} by separating the pretrained hosts from methods
that explicitly supply geometric information. The host rows describe the
policies to which we attach the graft---$\pi_{0.5}$ \cite{black2025pi05},
X-VLA \cite{zheng2026xvla}, Fast-WAM \cite{yuan2026fastwam}, and
LingBot-VA \cite{li2026lingbotva}---not competing 3D modules. Among the
spatial action models, SpatialVLA \cite{qu2025spatialvla} and GeoVLA
\cite{sun2025geovla} use depth-derived positions or
point clouds, while Spatial Forcing \cite{li2026spatialforcing} and GLaD
\cite{guo2025glad} distill frozen reconstruction
features during training. SERF \cite{kim2026serf} instead maintains a persistent map of the
scene and robot. These are different sources of spatial information:
explicit coordinates locate observations, distilled features provide a
learned geometric prior, and persistent maps retain information beyond the
current view. The ``training only'' annotation distinguishes supervision
that shapes the deployed policy from an additional geometric input supplied
at inference.

\paragraph{Where geometry enters}
The route column distinguishes methods even when they use similar signals.
SpatialVLA modifies visual-token positions, 3D-Mix \cite{yu2026threedmix} fuses reconstruction and
semantic features, and Spatial Forcing and GLaD align intermediate
representations. GeoVLA and PointVLA \cite{li2026pointvla} deliver point-cloud information to an
action expert. GeoPredict \cite{qian2026geopredict} uses auxiliary geometric
prediction, while DyWA \cite{lyu2025dywa}, GWM \cite{lu2025gwm},
WAM4D \cite{li2026wam4d}, X-WAM \cite{guo2026xwam}, and MECo-WAM
\cite{zhang2026mecowam} use geometric prediction or reconstruction to shape
action-relevant representations or world-model rollouts.
Our distinction is therefore not simply using 3D features or adding an
action-side module. We bind frozen features to metric, end-effector-relative
geometry and route the added residual through spatial banks in the late
action layers. This separates the source of the features from their
robot-relative grounding and from the mechanism by which they modify an
action. The bank interventions in Section~\ref{sec:ablations} test dependence
on that information, rather than attributing every benefit of an added
module to geometry.

\paragraph{Evaluation coverage and comparison limits}
The final column records the settings reported for each method, not a
common-protocol leaderboard. In particular, RoboTwin~1.0 is distinct from
RoboTwin~2.0, and a real-robot evaluation can involve different embodiments,
tasks, and trial counts. BEHAVIOR-1K coverage is also task-specific: SERF
\cite{kim2026serf}
reports three tasks, while our results report six individual tasks rather
than a benchmark-wide aggregate. The table highlights that we examine the
same grafting interface across tabletop simulation, distribution shift,
long-horizon mobile manipulation, and physical robots, using both VLA and
WAM hosts. It does not imply that every host was tested in every setting or
that broader coverage establishes higher performance. Quantitative
comparisons and their provenance are given in
Section~\ref{sec:experiments} and Section~\ref{app:eval-protocol}; the
controlled comparisons remain each host with and without its graft.

\section{Method}
\label{sec:method}

\subsection{Problem formulation}
\label{sec:problem}

A pretrained policy $\pi_\theta(A\mid I,\mathcal{T},S)$ maps $N$ RGB views
$I=\{I_v\}_{v=1}^{N}$, an instruction $\mathcal{T}$ and proprioceptive state
$S$ to an action chunk $A$. We consider hosts whose action expert is a
transformer trained by flow matching
\cite{lipman2023flow,black2025pi0,yuan2026fastwam} and re-run once per
integration step at inference, whether it reads a VLM trunk or a video
stream. Beyond the host's inputs, the graft assumes per-view calibration
$\mathcal{C}=\{(K_v,T_v)\}_{v=1}^{N}$ (intrinsics and camera-to-world pose
$T_v=(R_v,t_v)$ in one common frame), a metric depth map $D_v^{\star}$
measured by the sensor or predicted by the same frozen reconstruction model,
and at least one wrist camera per arm; no other input is assumed: no
segmentation, object state or map. From these \method builds one bank per
view, $B_v=B_v(I_v,\mathcal{C},D_v^{\star})$, and adds graft parameters
$\psi$, leaving the host intact:
$A\sim\pi_{\theta,\psi}(A\mid I,\mathcal{T},S,\{B_v\}_{v=1}^{N})$.

\subsection{Grounding the latent space: the spatial bank}
\label{sec:bank}

Binding happens at the grid: one bank token per feature-grid location, so the
correspondence between a latent and its metric point is never broken. The
result is a grounded latent space rather than a geometric representation: no
point is ever materialized as a cloud or map, and the tokens are rebuilt from
the current observation alone. For each view $v$ we tap frozen backbone
layers $\mathcal{L}_G$ on a common grid of $P$ locations. Taps differ in scale
by orders of magnitude, so each is normalized, projected, rescaled by a
learned gain and tagged with a layer embedding before fusion to
$\phi_{v,i}\in\mathbb{R}^{d}$, for $d$ the action-expert width.

\paragraph{Metric and robot-relative geometry} Back-projecting grid location
$i$ through $(K_v,T_v)$ and $D_v^{\star}$ yields a point $p_{v,i}$ and viewing
direction $q_{v,i}$ in the common calibration frame. Coordinates are
normalized as $\hat p=\operatorname{clip}((p-c)/\lambda,-1,1)$ about a
workspace centre $c$ with a single isotropic scale $\lambda$ per embodiment,
so that a $10$~cm displacement is encoded identically along every axis. With $e_1,\dots,e_E$ the
end-effector anchors, taken as the optical centres $e_k=t_k$ of the wrist
cameras (a fixed rigid offset from the gripper, so no kinematic model is
needed), each location carries
\begin{equation}
u_{v,i}
=
\left[
\gamma_B(\hat p_{v,i});
\left\{\gamma_B(\widehat{p_{v,i}-e_k})\right\}_{k=1}^{E};
q_{v,i};
m_{v,i}
\right],
\label{eq:spatial-vector}
\end{equation}
with $\gamma_B$ a Fourier encoding \cite{tancik2020fourier} and $m_{v,i}$
flagging valid depth. The absolute coordinate places the feature in the
workspace; the relative ones expose reaching and contact, and change as the
arms move even when the scene is static.

\paragraph{Binding} $\phi_{v,i}$ and $u_{v,i}$ describe the same location
and become one token: $u_{v,i}$ modulates $\phi_{v,i}$ through FiLM
\cite{perez2018film}, and the two are concatenated and projected into the
bank token $z_{v,i}\in\mathbb{R}^{d}$ (Section~\ref{app:method-details}), so
geometry enters both directly and by modulation.

\subsection{A bank-only graft for any flow-matching action expert}
\label{sec:injection}

Geometry enters only the final $M$ blocks of the action expert, whose states
already encode much of the base trajectory; the perceptual pathway receives
no graft. Flow matching motivates this placement: the sampler re-enters these
blocks at every integration step on a progressively cleaner chunk, so the
banks are read over exactly the interval in which a coarse action becomes
precise. $M$ spans the last quarter to third of the stack; grafting every
block is worse (Section~\ref{sec:ablations}). The
interface touches the host at exactly two points, the action-token states
$H^{(l)}$ and the residual update at block $l$; a new host changes only the
expert width $d$, the number of grafted blocks $M$ and the residual scale
$\rho$, plus a linear bridge when its blocks are wider than the injection
(Table~\ref{tab:param-cost}).

Each view $v$ has an independent multi-head cross-attention with queries $Q$
from the action states and keys and values from its bank $B$:
\begin{equation}
\begin{aligned}
\operatorname{CA}^{(l)}_v(Q,B)
&=W_{O}^{(l,v)}\operatorname{Attn}\Bigl(
\\[-2pt]
&\quad\nu\!\left(W_{Q}^{(l,v)}\operatorname{LN}(Q)\right),
\\[-2pt]
&\quad\nu\!\left(W_{K}^{(l,v)}\mathcal{N}(B)\right),\;
W_{V}^{(l,v)}\mathcal{N}(B)\Bigr),
\end{aligned}
\label{eq:crossattn}
\end{equation}
where $\mathcal{N}$ normalizes the key/value stream, $\nu$ normalizes queries
and keys per head, and the logits carry a clamped learned gain. The \emph{primary} view $v=1$, typically the global scene camera, writes
into the action stream first, with residual scale $\rho$:
$\bar H^{(l)}=H^{(l)}+\rho\operatorname{CA}^{(l)}_1(H^{(l)},B_1)$. Each
\emph{auxiliary} view then queries its own bank from this updated state
through a learned projection,
$A_v^{(l)}=\operatorname{CA}^{(l)}_v(P_v^{(l)}\bar H^{(l)},B_v)$. A single
bias-free linear projection $W_{\mathrm{merge}}^{(l)}$ fuses the concatenated
auxiliary outputs $[A_2^{(l)};\dots;A_N^{(l)}]$ back into the stream, and the
result $\tilde H^{(l)}$ replaces $H^{(l)}$ in each grafted block
(Section~\ref{app:derivations}, Algorithm~\ref{alg:graft}).

\paragraph{Bank-only routing} The projections $P_v^{(l)}$ build queries
only, and the merge consumes the attention outputs $A_v^{(l)}$ rather than
the branch activations, which would compose two linear layers into a
bank-independent action-state map. Zero banks then imply a zero graft delta
(Proposition~\ref{prop:bank-only}); we verify it empirically with bank-off rather than assume it
(Sections~\ref{app:proof} and~\ref{app:verification}). The graft adds $147$M trainable parameters to $\pi_{0.5}$ and X-VLA, $214$M
to Fast-WAM and $277$M to LingBot-VA, which is $4.4\%$ of the $3.3$B
$\pi_{0.5}$ and $5.4\%$ of the $5.1$B LingBot-VA (Table~\ref{tab:param-cost}).

\subsection{Training objective}
\label{sec:training-objective}
Each grafted host keeps its own flow-matching objective and finetuning
recipe; the graft only adds the banks to the conditioning set of the velocity
field. With noise $\epsilon\sim\mathcal{N}(0,I)$, level $\sigma\in[0,1]$,
interpolant $A_\sigma=(1-\sigma)A+\sigma\epsilon$ and the host's weighting
$w(\sigma)$,
\begin{equation}
\begin{aligned}
\mathcal{L}(\theta,\psi)
&=\mathbb{E}_{A,\epsilon,\sigma}\,w(\sigma)\\
&\quad\bigl\|v_{\theta,\psi}\bigl(A_\sigma,\sigma\mid I,\mathcal{T},S,\{B_v\}\bigr)
-(\epsilon-A)\bigr\|_2^2,
\end{aligned}
\label{eq:graft-loss}
\end{equation}
where $v_{\theta,\psi}$ is the host's velocity head evaluated with
$H^{(l)}\!\leftarrow\!\tilde H^{(l)}$ in the final $M$ blocks; there is no
auxiliary geometric supervision (Section~\ref{app:training}). Three scheduling choices are not neutral and are ablated in
Section~\ref{sec:ablations}: the output projections start at small non-zero
values, the host is held fixed for a brief warmup, and the graft residual is
dropped jointly across blocks with probability $p_{\mathrm{drop}}$.

\section{Method details}
\label{app:method-details}

\subsection{Derivations and full grafting equations}
\label{app:derivations}

\paragraph{Back-projection} Let $\tilde{x}_{v,i}$ be the homogeneous image
coordinate of location $i$, with normalized camera ray
$r_{v,i}=K_v^{-1}\tilde{x}_{v,i}/\|K_v^{-1}\tilde{x}_{v,i}\|_2$. Treating
$D^\star_{v,i}$ as optical-axis depth, the world-frame point and viewing
direction used in Section~\ref{sec:bank} are
\begin{equation}
p_{v,i}
=
R_v
\left(
\frac{D^\star_{v,i}}{r_{v,i,z}}\,r_{v,i}
\right)
+t_v,
\qquad
q_{v,i}
=
R_vr_{v,i}.
\label{eq:world-point}
\end{equation}

\paragraph{Injection} Writing $H^{(l)}$ for the action-token states at a
grafted block $l$, $\operatorname{CA}^{(l)}_v$ for the per-view
cross-attention of Eq.~\eqref{eq:crossattn}, $P_v^{(l)}$ for the view-specific
query projections and $\rho>0$ for the shared residual scale, the two stages
inside the injection module are
\begin{align}
\bar H^{(l)}
&=
H^{(l)}
+
\rho\,\operatorname{CA}^{(l)}_1\!\left(H^{(l)},B_1\right),
\label{eq:primary-injection}\\
A_v^{(l)}
&=
\operatorname{CA}^{(l)}_v
\!\left(
P_v^{(l)}\bar H^{(l)},\,B_v
\right),
\qquad
v=2,\ldots,N,
\label{eq:aux-attention}\\
\tilde H^{(l)}
&=
\bar H^{(l)}
+
\rho\,
W_{\mathrm{merge}}^{(l)}
\operatorname{Concat}_{v=2}^{N}
A_v^{(l)},
\label{eq:multi-view-injection}
\end{align}

The merge is a single bias-free linear projection of the attention outputs
$A_v^{(l)}$, not the branch activations $P_v^{(l)}\bar H^{(l)}$; with two
auxiliary wrist views its input has width $2d$, and with one it has width
$d$. The module returns
$\Delta^{(l)}=\tilde H^{(l)}-H^{(l)}$, and the caller applies
$H^{(l)\prime}=H^{(l)}+g_l\kappa\Delta^{(l)}$. Here $g_l$ selects the final
$M$ blocks and whole-graft dropout draws one
$\kappa\sim\mathrm{Bernoulli}(1-p_{\mathrm{drop}})$ per sample per step,
shared across those blocks; kept samples are not rescaled.

\paragraph{Parameter cost} Each grafted block holds one primary
cross-attention, $N-1$ query projections, $N-1$ auxiliary cross-attentions and
one merge, and the bank builder is shared by all of them, so
\begin{equation}
|\psi|
\;=\;
\underbrace{M\,(6N-2)\,d^{2}}_{\text{grafted blocks}}
\;+\;
\underbrace{|\psi_{\mathrm{bank}}|}_{\text{bank builder}}
\;+\;
|\psi_{\mathrm{bridge}}|,
\label{eq:param-count}
\end{equation}
with $d$ the injection width. With $N=3$ views and $d=1024$, each grafted
block costs $16d^{2}=16.8$M, so the injection grows linearly in the number of
grafted blocks $M$ and quadratically in $d$. The bank builder is $46.0$M for
every host, and $|\psi_{\mathrm{bridge}}|$ is non-zero only when the host's
blocks are wider than the injection width. Table~\ref{tab:param-cost} gives
the measured counts.

\begin{table*}[t]
\centering
\small
\caption{\textbf{Trainable parameters added by the graft on RoboTwin.} Counts
are measured by building the modules. The share is relative to the host's own
parameter count.}
\label{tab:param-cost}
\begin{tabular}{@{}lccccc@{}}
\toprule
Host & Grafted blocks & Bank builder & Injection & Bridges & Total (share) \\
\midrule
$\pi_{0.5}$ & 6 & 46.0M & 100.7M & -- & 146.8M ($4.4\%$ of 3.3B) \\
X-VLA & 6 & 46.0M & 100.7M & -- & 146.8M \\
$\pi_{0.5}$, VGGT-$\Omega$ & 6 & 48.1M & 100.7M & -- & 148.9M \\
Fast-WAM & 10 & 46.0M & 168.0M & -- & 214.0M \\
LingBot-VA & 10 & 46.0M & 168.0M & 62.9M & 276.9M ($5.4\%$ of 5.09B) \\
\bottomrule
\end{tabular}
\end{table*}

\begin{algorithm*}[t]
\caption{\method: bank construction and grafting at one denoising step.
The exact-zero claim of Proposition~\ref{prop:bank-only} is conditional on
bias-free bank-dependent paths.}
\label{alg:graft}
\begin{algorithmic}[1]
\Require RGB images $\{I_v\}_{v=1}^{N}$, calibration $\{(K_v,T_v)\}$, metric depth
$\{D_v^\star\}$, end-effector anchors $\{e_k\}_{k=1}^{E}$, host action states
$H^{(l)}$, frozen spatial model $G$, grafted-block count $M$
\Statex \textbf{Bank construction} \Comment{once per observation}
\For{$v=1,\ldots,N$}
  \State $\{F_v^{(l)}\}_{l\in\mathcal{L}_G}\gets G(I_v)$
         \Comment{frozen multi-layer taps}
  \State $\phi_v\gets W_{\mathrm{fuse}}\operatorname{Concat}_l
         (\operatorname{LN}(P_l(F_v^{(l)}))+\eta_l)$
  \State $p_v,q_v\gets$ back-project $D_v^\star$ through $(K_v,T_v)$
         \Comment{Eq.~\eqref{eq:world-point}}
  \State $u_v\gets[\gamma_B(\hat p_v);\{\gamma_B(\widehat{p_v-e_k})\}_{k};q_v;m_v]$
         \Comment{Eq.~\eqref{eq:spatial-vector}}
  \State $s_v\gets\mathrm{MLP}(u_v)$;\quad
         $f'_v\gets(\mathbf{1}+\Gamma(s_v))\odot\phi_v+\beta(s_v)$
         \Comment{FiLM binding}
  \State $z_v\gets\mathrm{LN}(W_{\mathrm{out}}[s_v;f'_v])+\text{grid and view
         embeddings}$
  \State $B_v\gets z_v-\operatorname{mean}_{\text{tokens}}(z_v)$
         \Comment{per-sample centering}
\EndFor
\Statex \textbf{Injection into the final $M$ blocks}
\State draw $\kappa\sim\mathrm{Bernoulli}(1-p_{\mathrm{drop}})$
       \Comment{one draw shared by every block}
\For{$l=L-M+1,\ldots,L$}
  \State $\bar H\gets H^{(l)}+\rho\,\operatorname{CA}^{(l)}_1(H^{(l)},B_1)$
         \Comment{primary view}
  \State $A_v\gets\operatorname{CA}^{(l)}_v(P_v^{(l)}\bar H,B_v)$ for
         $v=2,\ldots,N$ \Comment{$P_v^{(l)}$ builds queries only}
  \State $\tilde H\gets\bar H+\rho\,W_{\mathrm{merge}}^{(l)}
         \operatorname{Concat}_{v\ge2}A_v$
         \Comment{merges attention \emph{outputs}}
  \State $H^{(l)\prime}\gets H^{(l)}+\kappa\,(\tilde H-H^{(l)})$
\EndFor
\end{algorithmic}
\end{algorithm*}

\paragraph{Frame parameters} The workspace centre $c$ and isotropic
scale $\lambda$ define the frame the graft reasons in, so they are fitted once
per embodiment on the training split and stored with the checkpoint. The
hand-relative encodings use their own fitted centres, distinct from the
workspace centre. Where the end-effector anchors are taken as wrist-camera
optical centres, each differs from the true end-effector position by a fixed
rigid offset, which the geometry encoder absorbs.

\subsection{The bank-only invariant}
\label{app:proof}

\begin{proposition}[Bank-only invariant]
\label{prop:bank-only}
If every projection on a bank-dependent path
($W_K$, $W_V$, $W_O$, $W_{\mathrm{merge}}$) is bias-free and the key/value
normalization $\mathcal{N}$ is scale-only, then $B_1=\cdots=B_N=0$ implies
$H^{(l)\prime}=H^{(l)}$ at every grafted block $l$.
\end{proposition}

\begin{proof}
Scale-only normalization gives $\mathcal{N}(0)=0$, so $B_v=0$ yields zero keys
and values. The attention output is a convex combination of zero value vectors,
hence $0$ for any query, and the bias-free output projection maps it to $0$.
Thus $\bar H^{(l)}=H^{(l)}$, and $A_v^{(l)}=0$ for $v\ge2$ regardless of
$P_v^{(l)}\bar H^{(l)}$, so the bias-free merge projection maps the
concatenation of zeros to zero and $H^{(l)\prime}=H^{(l)}$.
\end{proof}

\paragraph{Implementation status} The $\pi_{0.5}$ implementation uses the
attention-output merge, eliminating the direct linear path from an auxiliary
query projection to the action stream. Setting
\texttt{spatial\_no\_bias=True} removes biases from the auxiliary attention
output projections and the merge, but the primary attention output projection,
key/value projections and key/value LayerNorm still have trainable biases.
Their biases initialize at zero, so a zero bank gives a zero delta at
initialization; Proposition~\ref{prop:bank-only} is not guaranteed for the
trained checkpoint. A nonzero bias can emit a scene-independent constant even
with a zero bank. We therefore interpret bank-off as a measurement of the
trained model with its banks removed, not as an assumed-zero residual.

\subsection{Architectural verification}
\label{app:verification}

Proposition~\ref{prop:bank-only} establishes an exact-zero invariant only
under its stated bias and normalization conditions. Before training, we check
numerically that a zero spatial bank gives zero graft residual and a nonzero
bank gives a nonzero one. This verifies initialization, not that the invariant
continues to hold after the remaining biases have trained. For the Fast-WAM
instantiation, strict checkpoint loading verifies the frozen host snapshot and
every added bank-builder and grafting parameter. The evaluation preflight
additionally checks view ordering, RGB and depth preprocessing, calibrated
intrinsics and extrinsics, the 32-action prediction and 24-action execution
contract, and live construction of a nonzero bank from the current three-view
observation. Measured depth is retained at millimeter precision and aligned to
the $18\times24$ DA3 feature grid with nearest-neighbor resampling.

These checks separate initialization and wiring from downstream empirical
performance. Bank-off measures the trained graft with its banks removed;
bank-shuffle measures whether its gain depends on the banks' sample-specific
geometry (Section~\ref{sec:ablations}).

\section{Experiments}
\label{sec:experiments}

\paragraph{Hosts and comparison protocol} We apply \method to four hosts chosen to span
both families of flow-matching policy: two VLAs, $\pi_{0.5}$ and X-VLA, and
two WAMs, Fast-WAM and LingBot-VA, and the graft is identical across every host. The reconstruction
backbone is DA3, kept frozen and tapped at four layers on a common feature grid, and
the spatial banks enter the final six action-expert blocks of the VLAs and
the final ten of the WAMs; the only host-specific settings are the ones
listed in Section~\ref{sec:injection}.

\paragraph{Benchmarks and baselines} We evaluate on four simulation benchmarks and three real robots, each under its standard protocol and each
chosen to stress a different demand. LIBERO (130 tasks, 50 seeds per task)
is near saturated and tests whether the graft leaves a working policy
intact. RoboTwin~2.0 (50 dual-arm tasks, 20 seeds per task in each of the
Clean and Randomized configurations) tests \method's manipulation capability and robustness to appearance changes. RoboPRO (80 tasks, 20 seeds per task in
clean and cluttered scenes, scored under both its Easy and Hard criteria)
adds clutter objects in the robot's path and its success metric penalizes collisions. B1K (six tasks, 20 public evaluation instances each) is the most difficult manipulation benchmark we test: long-horizon mobile manipulation that chains sub-skills together within a large household environment, and shows the breadth of \method's application. The real-robot trials (10 per task) test whether the gains survive measured depth and real
calibration error. Success rates (SR) are reported in \%, and every reported
change is an absolute difference. Beyond each host's own base, we compare
against the published 3D-conditioned policies listed in
Table~\ref{tab:libero-robotwin} and, on B1K, against SERF and the two
leading entries of the 2025 challenge. Host configurations, training data,
depth sources and the provenance of every quoted number are documented in
Sections~\ref{app:training} and~\ref{app:eval-protocol}, and per-task
results in Section~\ref{app:benchmark-details}.

We ask three questions, one per contribution: does one graft improve both
VLAs and WAMs (Section~\ref{sec:results}); do the gains hold under clutter,
long horizons and real sensing (Section~\ref{sec:complex-settings}); and do
they come from the grounded geometry and its placement rather than added
parameters (Section~\ref{sec:ablations})?

\subsection{One graft improves every host}
\label{sec:results}

\begin{table*}[t]
\centering
\footnotesize
\setlength{\tabcolsep}{1pt}
\renewcommand{\arraystretch}{1.02}
\caption{\textbf{LIBERO} \textbf{and RoboTwin SR (\%).}
Above the line: published 3D-conditioned policies, quoted from their own
papers, except Spatial Forcing on RoboTwin, which we trained ourselves
(Section~\ref{app:eval-protocol}). Below the line: our four hosts. LIBERO bases are quoted from each host's
published results; the $\pi_{0.5}$ entry is from the openpi release
\cite{openpi2025}, as its paper reports no LIBERO result. On RoboTwin, the $\pi_{0.5}$ base is
cited from \cite{li2026lingbotva} and the X-VLA base from \cite{bi2026motus},
both trained on the same RoboTwin demonstrations as our grafts, so the pairs
are directly comparable; the Fast-WAM and LingBot-VA bases are our
evaluations of their public checkpoints with our seeds and pipeline. Every
grafted entry is our run. Green/red: absolute change from the host's own
base. NR: not reported in the cited paper.
}
\label{tab:libero-robotwin}
\settowidth{\gw}{\textbf{94.0}\dpos{11.3}}
\begin{tabular*}{\linewidth}{@{\extracolsep{\fill}}l B@{\extracolsep{0pt}\,$\rightarrow$\,}G@{\extracolsep{\fill}} B@{\extracolsep{0pt}\,$\rightarrow$\,}G@{\extracolsep{\fill}} B@{\extracolsep{0pt}\,$\rightarrow$\,}G@{}}
\toprule
Model & \multicolumn{2}{c}{LIBERO Avg.} & \multicolumn{2}{c}{RoboTwin Clean}
& \multicolumn{2}{c}{RoboTwin Rand.} \\
\midrule
GLaD & \multicolumn{2}{c}{94.1} & \multicolumn{2}{c}{NR} & \multicolumn{2}{c}{NR} \\
GeoVLA & \multicolumn{2}{c}{97.7} & \multicolumn{2}{c}{NR} & \multicolumn{2}{c}{NR} \\
SpatialVLA & \multicolumn{2}{c}{78.1} & \multicolumn{2}{c}{NR} & \multicolumn{2}{c}{NR} \\
Spatial Forcing & \multicolumn{2}{c}{98.5} & \multicolumn{2}{c}{66.2} & \multicolumn{2}{c}{58.9} \\
WAM4D & \multicolumn{2}{c}{NR} & \multicolumn{2}{c}{\underline{93.8}} & \multicolumn{2}{c}{89.9} \\
\midrule
$\pi_{0.5}\rightarrow\pi_{0.5}^{\mathrm{graft}}$
  & 96.9 & \textbf{99.6}\dpos{2.7}
  & 82.7 & \textbf{94.0}\dpos{11.3}
  & 76.8 & \textbf{92.4}\dpos{15.6} \\
X-VLA $\rightarrow$ X-VLA$^{\mathrm{graft}}$
  & 98.1 & \underline{98.7}\dpos{0.6}
  & 72.8 & 83.7\dpos{10.9}
  & 72.8 & 82.6\dpos{9.8} \\
Fast-WAM $\rightarrow$ Fast-WAM$^{\mathrm{graft}}$
  & 97.6 & 96.9\dneg{0.7}
  & 82.6 & 86.0\dpos{3.4}
  & 81.8 & 83.3\dpos{1.5} \\
LingBot-VA $\rightarrow$ LingBot-VA$^{\mathrm{graft}}$
  & 98.5 & 97.4\dneg{1.1}
  & 88.9 & 93.3\dpos{4.4}
  & 85.7 & \underline{91.4}\dpos{5.7} \\
\bottomrule
\multicolumn{7}{@{}l}{\emph{\method: base $\rightarrow$ graft}} \\
\end{tabular*}
\end{table*}

\paragraph{RoboTwin: consistent gains, robust to visual change}
RoboTwin's paired Clean and Randomized configurations use the same dual-arm
tasks, separating task capability from robustness to scene and appearance
changes. All four hosts improve, by $+7.5\%$ on Clean and $+8.2\%$ on
Randomized on average. The VLAs gain most ($+9.8\%$ to $+15.6\%$),
while the WAMs also benefit despite their stronger base scores. Every host also improves on Randomized scenes, by $+1.5\%$ to $+15.6\%$,
although the reconstruction backbone is frozen and is not fine-tuned on RoboTwin's
randomized clutter, lighting or backgrounds: geometry learned from diverse
scenes carries across appearance changes the host has not seen. Qualitatively, the mug
column of Figure~\ref{fig:qualitative-results} shows the mechanism: the
grafted policy estimates the rack height well enough to clear the hook, where the base approaches at the wrong height. At the task level, the largest per-task gain for $\pi_{0.5}$ is Move can pot, $51\%\to100\%$ Clean and $55\%\to95\%$ Randomized (base per-task rates from \cite{li2026lingbotva}): the can must be set down at a position defined relative to the pot, which is exactly the robot-relative metric placement the bank encodes. Blocks ranking size, which needs the blocks' physical sizes, is next ($49\%\to85\%$, $26\%\to70\%$).

Grafted $\pi_{0.5}$ reaches $94.0\%/92.4\%$, above the best published
3D-conditioned policy, WAM4D ($93.8\%/89.9\%$), and our full-benchmark Spatial
Forcing run ($66.2\%/58.9\%$) is below every grafted host. These cross-paper rows
are not host-matched; the controlled result is that all four base/graft pairs
improve. Per-task results for every policy we evaluated ourselves are in
Section~\ref{app:benchmark-details}, Table~\ref{tab:robotwin-per-task}.

\paragraph{LIBERO: maintained performance on a saturated benchmark} Every
base already scores $96.9$--$98.5$, so this benchmark can only show whether
the graft integrates well into a working policy. The two VLAs improve
($+2.7\%$ for $\pi_{0.5}$, $+0.6\%$ for X-VLA) and the two WAMs slip by about
$1\%$ ($-0.7\%$ and $-1.1\%$), with all final scores remaining in the near-saturated range
(Table~\ref{tab:libero-robotwin}). Grafted $\pi_{0.5}$ is the highest entry at
$99.6\%$, but a $1.1\%$ lead on a saturated benchmark is not evidence about
geometry; we treat LIBERO as a compatibility check of one frozen
reconstruction model across four action experts and draw no conclusion about
geometry from it. Per-suite results are in Section~\ref{app:benchmark-details},
Table~\ref{tab:libero-details}.

\subsection{Gains hold under clutter, long horizons and real sensing}
\label{sec:complex-settings}

\begin{table*}[t]
\centering
\small
\setlength{\tabcolsep}{7pt}
\renewcommand{\arraystretch}{1.15}
\caption{\textbf{B1K Q-score on six tasks.}
Q-score is average task progress. $\pi_{0.5\text{-}\mathrm{RLC}}$ and $\pi_{0.5\text{-}\mathrm{Comet}}$ are the
\#1 and \#2 entries of the 2025 challenge \cite{behaviorchallenge2025} and
SERF \cite{kim2026serf} is quoted on the three tasks it reports; all three
are taken from their published scores. The grafted policy is evaluated on the
20 public evaluation instances of each task; official challenge scoring uses only 10 of those 20. Green/red: change from the host's own base.}
\label{tab:behavior-five}
\setgw{\textbf{0.714}\dpos{0.471}}{\boldmath$\pi_{0.5\text{-}\mathrm{RLC}}^{\mathrm{graft}}$}
\begin{tabular}{@{}l cc B@{\,$\rightarrow$\,}G@{}}
\toprule
\textbf{Task} & \textbf{SERF}
& \boldmath$\pi_{0.5\text{-}\mathrm{Comet}}$
& \boldmath$\pi_{0.5\text{-}\mathrm{RLC}}$ & \boldmath$\pi_{0.5\text{-}\mathrm{RLC}}^{\mathrm{graft}}$ \\
\midrule
Assembling gift baskets & \underline{0.525} & 0.312 & 0.281 & \textbf{0.632}\dpos{0.351} \\
Collecting children's toys & \underline{0.635} & 0.583 & 0.500 & \textbf{0.673}\dpos{0.173} \\
Putting shoes on rack & \underline{0.601} & 0.470 & \textbf{0.720} & 0.590\dneg{0.130} \\
Putting up Christmas decorations & NR & 0.356 & \underline{0.478} & \textbf{0.510}\dpos{0.032} \\
Clean boxing gloves & NR & 0.000 & \underline{0.200} & \textbf{0.600}\dpos{0.400} \\
Cleaning up plates and food & NR & 0.043 & \underline{0.243} & \textbf{0.714}\dpos{0.471} \\
\bottomrule
\multicolumn{5}{@{}l}{\emph{\method: base $\rightarrow$ graft}} \\
\end{tabular}
\end{table*}

\paragraph{B1K: largest gains where grasping is hardest}
Of the spatial action models we compare, only SERF~\cite{kim2026serf}
reports B1K results, on three tasks. Grafting the challenge-winning
$\pi_{0.5\text{-}\mathrm{RLC}}$~\cite{larchenko2025rlc} increases average
task progress (Q-score; Table~\ref{tab:behavior-five}); we evaluate on the
20 public evaluation instances of each task, whereas official challenge scoring uses only 10 of those 20. The graft improves five of the six tasks, by $+0.216$ Q-score on average and
up to $+0.471$; the exception is putting shoes on rack ($0.720\to0.590$).
B1K chains minute-long mobile activities without resetting between contacts,
so a small error in one action chunk leaves the robot in a state its
demonstrations rarely cover, and the next chunk starts further out of
distribution. We attribute the graft's gains at this horizon to the same
precision it adds on the tabletop: each contact lands closer to where the
demonstrations put it, so there is less error to compound.

SERF \cite{kim2026serf} builds on the same $\pi_{0.5\text{-}\mathrm{RLC}}$
checkpoint but addresses a different failure: its persistent map gives a
memoryless policy access to information about objects that have left its view (Section~\ref{sec:related}),
whereas the graft targets precision at contact and keeps no state across
observations. On the three tasks SERF reports, the grafted policy averages $0.632$ against
SERF's reported $0.587$; the margin is modest because those tasks are
dominated by the failure SERF addresses.

The graft's largest gains come where the policy fails at the grasp itself: on
the other three tasks it improves its base by $+0.301$ on average. Putting up
Christmas decorations requires picking up a thin, smooth candy cane among the
wire stems of a wreath; in Figure~\ref{fig:qualitative-results} the base
closes beside the shaft while the grafted policy closes on it. Its Q-score
gain is small ($0.478\to0.510$) because the base already completes much of
the activity. Clean boxing gloves has rounded gloves that can be grasped only
at the cuff ($0.200\to0.600$), and cleaning up plates and food requires
lifting flat plates and bowls ($0.243\to0.714$). The graft helps most where the grasp is hardest, consistent with its design
objective.

\begin{figure*}[t]
\centering
\begingroup
\setlength{\fboxsep}{0pt}
\setlength{\fboxrule}{0.35pt}

\newlength{\qualpanelheight}
% Four 6:5 views plus one square view, with four fixed 4pt gaps.
\setlength{\qualpanelheight}{\dimexpr(\linewidth-16pt-10\fboxrule)*10/58\relax}
\newcommand{\compactframe}[1]{\fcolorbox{black!45}{white}{\includegraphics[height=\qualpanelheight]{figures/qualitative_results/compact_#1.png}}}
\newcommand{\resultcolumn}[3]{%
  \begin{minipage}[t]{#1}\centering
  {\scriptsize #2\par}\vspace{2pt}%
  \compactframe{#3_success}\par\nointerlineskip\vspace{2pt}%
  \compactframe{#3_failure}\par
  \end{minipage}}
\noindent
\resultcolumn{\dimexpr\qualpanelheight*6/5+2\fboxrule\relax}{R1Pro\\orange on plate}{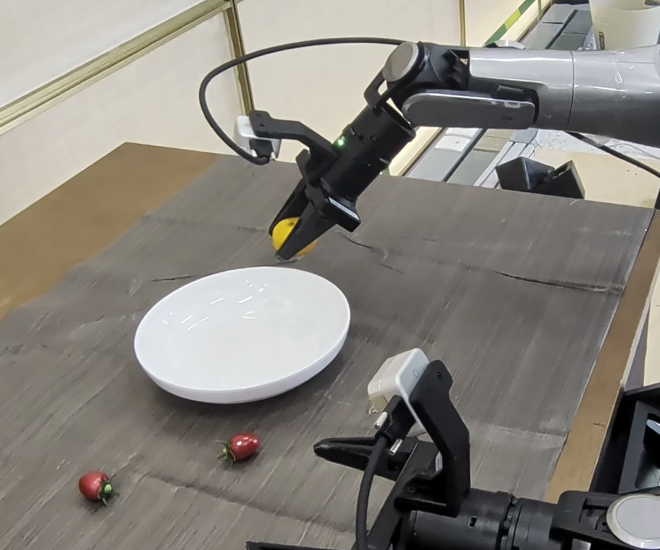}\hspace{4pt}%
\resultcolumn{\dimexpr\qualpanelheight*6/5+2\fboxrule\relax}{Piper\\plate on rack}{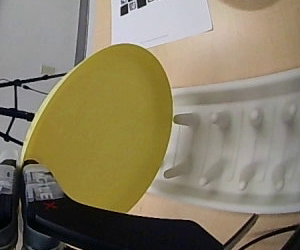}\hspace{4pt}%
\resultcolumn{\dimexpr\qualpanelheight+2\fboxrule\relax}{B1K\\pick up candy}{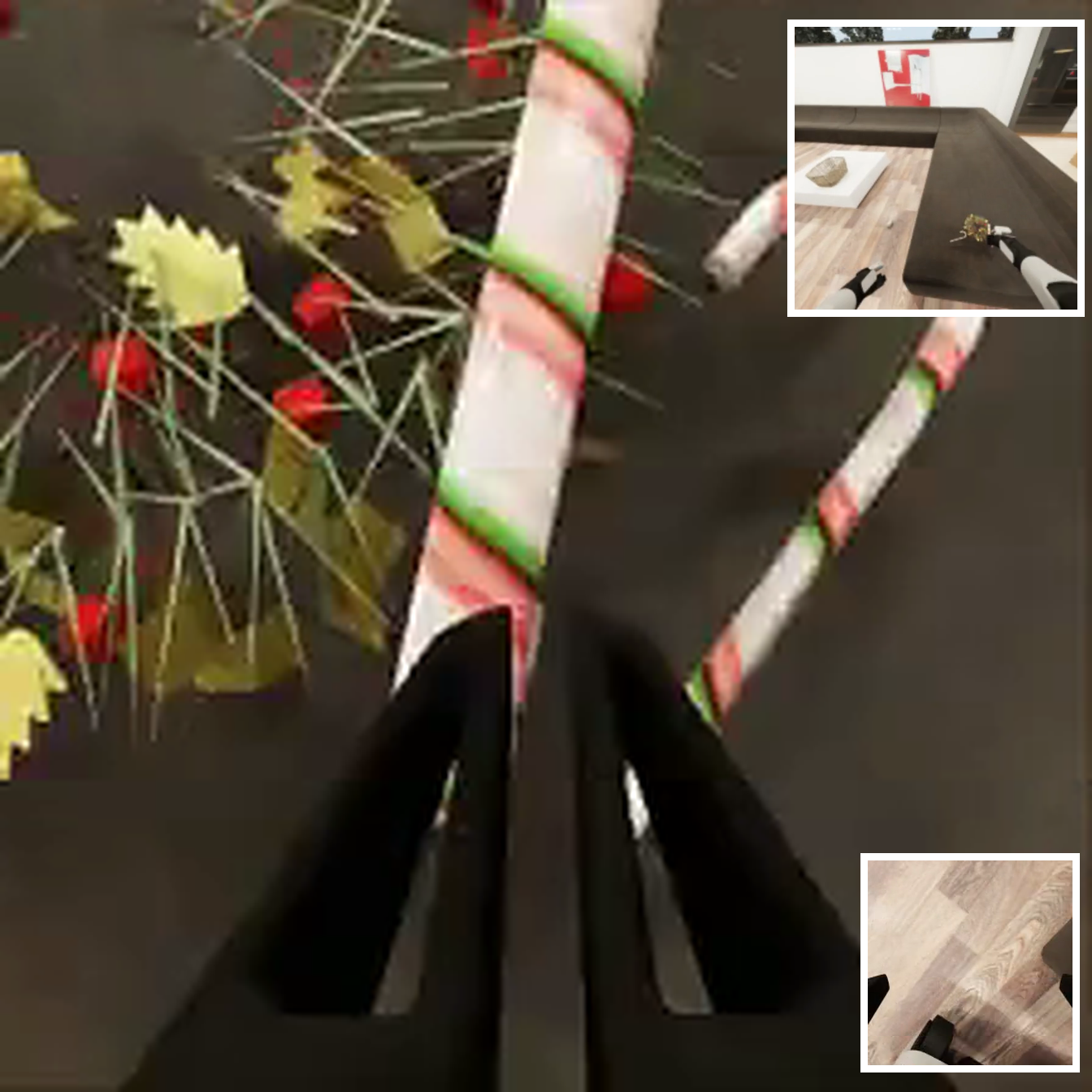}\hspace{4pt}%
\resultcolumn{\dimexpr\qualpanelheight*6/5+2\fboxrule\relax}{RoboTwin\\hang mug}{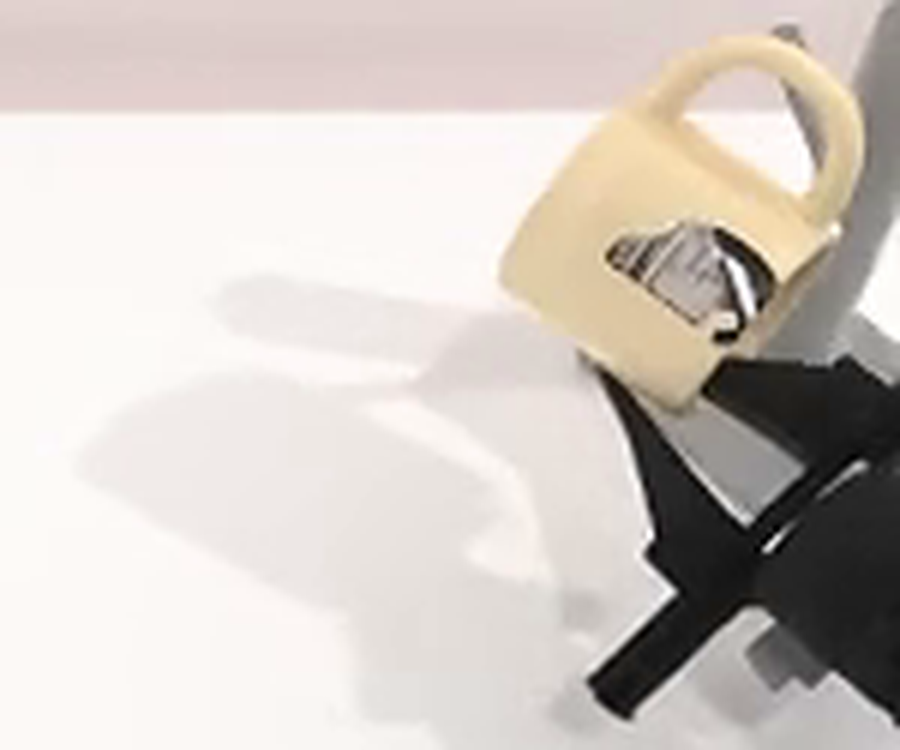}\hspace{4pt}%
\resultcolumn{\dimexpr\qualpanelheight*6/5+2\fboxrule\relax}{RoboPRO\\plate in sink}{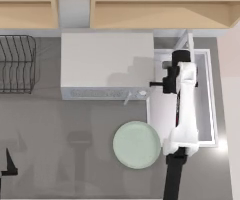}
\endgroup
\caption{\textbf{Grafted success (top) against base failure (bottom), one
column per benchmark.} Each column is the task discussed in the
corresponding paragraph. R1Pro: the grafted policy estimates plate height and
releases the orange with clearance, where the base drops it short. Piper:
plate--rack alignment avoids the collision the base drives into.
B1K: the base shuts its jaws beside the candy cane's shaft; the
grafted policy closes on it. RoboTwin: rack height is estimated well
enough to clear the hook. RoboPRO: a lower, better-placed grasp seats the
plate in the sink instead of catching its rim.}
\label{fig:qualitative-results}
\end{figure*}

\begin{table*}[t]
\centering
\footnotesize
\setlength{\tabcolsep}{1pt}
\renewcommand{\arraystretch}{1.15}
\caption{\textbf{RoboPRO success (\%).} Each cell reads Easy/Hard. Easy is
RoboPRO's success rate (SR); Hard is its stricter harsh success rate (HSR),
scored on the same episodes.
Bases are reported by \cite{li2026robopro} and grafted scores are ours. Green: change from the host's own base.}
\label{tab:robopro}
\settowidth{\gw}{\textbf{78.8}\,/\,\textbf{73.7}}
\begin{tabular*}{\linewidth}{@{\extracolsep{\fill}}l B@{\extracolsep{0pt}\;$\rightarrow$\;}G@{\extracolsep{\fill}\hspace{40pt}}B@{\extracolsep{0pt}\;$\rightarrow$\;}G@{\hspace{46pt}}}
\toprule
Model & \multicolumn{2}{c}{Clean (Easy/Hard)} & \multicolumn{2}{c}{Clutter (Easy/Hard)} \\
\midrule
$\pi_{0.5}\rightarrow\pi_{0.5}^{\mathrm{graft}}$
  & \underline{70.3}\,/\,\underline{65.7} & \textbf{78.8}\,/\,\textbf{73.7}\rlap{\dpos{8.5/{+}8.0}}
  & \underline{60.9}\,/\,\underline{47.8} & \textbf{69.8}\,/\,\textbf{49.4}\rlap{\dpos{8.9/{+}1.6}} \\
X-VLA\,$\rightarrow$\,X-VLA$^{\mathrm{graft}}$
  & 49.5\,/\,46.7 & 60.9\,/\,49.9\rlap{\dpos{11.4/{+}3.2}}
  & 39.7\,/\,27.6 & 45.3\,/\,35.3\rlap{\dpos{5.6/{+}7.7}} \\
\bottomrule
\multicolumn{5}{@{}l}{\emph{\method: base $\rightarrow$ graft}} \\
\end{tabular*}
\end{table*}

\paragraph{RoboPRO: improved collision avoidance in clutter} Both hosts improve
in every condition on both Easy and Hard (Table~\ref{tab:robopro}). Cluttered scenes place distractors inside the path the demonstration
sweeps, and the Hard rate additionally penalizes collisions with them. Hard
rises together with Easy in every cell, so the grafted policies complete
more tasks without colliding:
under clutter, $\pi_{0.5}$ gains $+8.9\%$ Easy and $+1.6\%$ Hard, and X-VLA
gains $+5.6\%$ Easy and $+7.7\%$ Hard, more than its clean-scene Hard gain of
$+3.2\%$. Knowing where each distractor sits relative to the grippers lets the policy
route around it; the sink column of Figure~\ref{fig:qualitative-results}
shows a grasp low and square enough to seat the plate rather than catch its
rim. Per-task results for all 80 tasks are in Section~\ref{app:benchmark-details},
Table~\ref{tab:robopro-per-task}.

\begin{table*}[t]
\centering
\small
\setlength{\tabcolsep}{3pt}
\renewcommand{\arraystretch}{1.05}
\caption{\textbf{Real-robot evaluation (SR, \%).} Matched base/graft
comparisons on three platforms, with 10 trials per task. A trial is considered successful only if all sub-tasks are completed. Green: change from
the base.}
\label{tab:real-robot}
\settowidth{\gw}{80\dpos{50}}
\begin{tabularx}{\linewidth}{@{}lX B@{\,$\rightarrow$\,}G c@{}}
\toprule
Platform & Task & \multicolumn{2}{c}{SR} & Trials \\
\midrule
UR5e & Pick up marker and place in pen cup (thin objects, constrained placement) & 30 & \textbf{80}\dpos{50} & 10 \\
\midrule
\multirow{3}{*}{Piper} & Clean up table (transparent objects, long horizon) & 20 & \textbf{70}\dpos{50} & 10 \\
& Open tea can (bimanual dexterous manipulation) & 10 & \textbf{40}\dpos{30} & 10 \\
& Place plate on rack (thin objects, constrained placement) & 70 & \textbf{90}\dpos{20} & 10 \\ \midrule
\multirow{2}{*}{R1Pro} & Pick up tangerine and place on plate (thin objects, constrained placement) & 50 & \textbf{90}\dpos{40} & 10 \\
& Clean up table (deformable object, long horizon) & 50 & \textbf{80}\dpos{30} & 10 \\
\bottomrule
\multicolumn{5}{@{}l}{\emph{\method: base $\rightarrow$ graft}} \\
\end{tabularx}
\end{table*}

\paragraph{Real robots: gains transfer to physical sensing}
The real-robot trials test whether the gains survive measured depth and real
calibration error, on three platforms of different form: a single-arm UR5e,
a bimanual pair of AgileX Pipers and the humanoid Galaxea R1Pro. Grafting
improves success on all six tasks, by $20\%$ to $50\%$
(Table~\ref{tab:real-robot}), and the largest gains come where the base fails
on depth at contact: marker grasping ($30\%\to80\%$, approach at the wrong
height), Piper table cleanup ($20\%\to70\%$, transparent objects missed and a
white mug confused with the background) and R1Pro tangerine placement
($50\%\to90\%$, fruit dropped short of the plate). Plate-to-rack placement
improves from $70\%$ to $90\%$ through better rack alignment
(Figure~\ref{fig:qualitative-results}; full-size rollouts in Figure~\ref{fig:qualitative-results-full}) and deformable-object cleanup from
$50\%$ to $80\%$. Tea-can opening remains hardest ($10\%\to40\%$): lid-depth
errors are reduced, but the task also demands bimanual coordination that
geometry alone does not supply. That the graft carries to hardware without
any adaptation of its geometry pathway is expected: DA3 is trained on large
and diverse real-world imagery, so its reconstruction is, if anything, better
matched to real scenes than to simulated ones, and the grounded latents the
graft reads are the same on either side of the simulation boundary.

\subsection{The gain comes from the grounded geometry and its placement}
\label{sec:ablations}

\begin{figure*}[t]
\centering
\includegraphics[width=\linewidth]{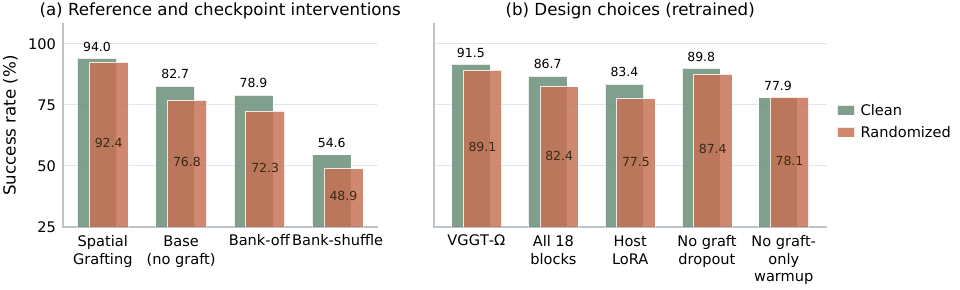}
\caption{\textbf{Ablations on RoboTwin with $\pi_{0.5}$ (SR \%).}
(a) Reference, ungrafted base, and bank interventions on a fixed checkpoint.
(b) One design choice retrained at a time. Equal-width Clean (green) and
Randomized (terracotta) bars overlap at each configuration and share a 25\%
baseline. }
\label{fig:ablations}
\end{figure*}

\paragraph{Bank content, not capacity, carries the gain} Figure~\ref{fig:ablations}(a) intervenes
on a fixed $\pi_{0.5}$ checkpoint on RoboTwin. Zeroing the banks cuts
Clean/Randomized by $15.1\%/20.1\%$ to $78.9\%/72.3\%$, below the ungrafted
base. Substituting another episode's banks cuts a further $24.3\%/23.4\%$ to $54.6\%/48.9\%$. Weights and host inputs stay fixed: mismatched bank content
hurts more than absent banks, and bank-only routing removes a direct
action-state bypass.

\paragraph{The backbone is interchangeable; late injection is essential}

The $\pi_{0.5}$ RoboTwin reference uses frozen DA3 features at four
levels, the complete spatial grid, and injection into the final six
action-expert blocks; host-specific configurations are in
Section~\ref{app:training}. Panel (b) retrains one design change at a time
from this reference, which also uses dropout and warmup. Replacing
DA3~\cite{lin2026depthanything3} with VGGT-$\Omega$
\cite{wang2026vggtomega}, everything else unchanged, reaches $91.5\%$ on Clean and $89.1\%$ on Randomized. This is $2.5\%$ and $3.3\%$ below the reference but $8.8\%$ and $12.3\%$ above
the base: the
interface transfers across reconstruction backbones, while its ceiling still
depends on their quality. Conditioning all $18$ blocks instead of the final six
costs $7.3\%$ and $10.0\%$, supporting late injection after the host has
formed a task-conditioned action.

\paragraph{Stable adaptation needs warmup and dropout; LoRA is not enough} LoRA gives $83.4\%/77.5\%$, within $0.7\%$ of the base on both configurations, whereas full finetuning delivers the graft.
Training without whole-graft dropout ($p_{\mathrm{drop}}=0$) costs $4.2\%$ and $5.0\%$, showing the value of periodically acting without the spatial
residual. Skipping the graft-only warmup and adapting the host from the first
update costs $16.1\%$ and $14.3\%$, dropping Clean below the base. This is
the largest single design effect and helps explain why schedules need to vary
with the pretrained host.

Across all variants, the interface's main effect is to close the
Clean-minus-Randomized gap (Section~\ref{app:training}).

\section{Reference graft configurations}
\label{app:training}

This section records the graft configuration used for the reported RoboTwin
results.

\paragraph{Objective} The loss of Eq.~\eqref{eq:graft-loss} is the host's own
flow-matching objective with the banks added to its conditioning set; the
noise distribution, weighting and action parameterization are the host's
own, and there is no auxiliary
geometric supervision, so a grafted run differs from its baseline only in
the presence of $\{B_v\}$.

\paragraph{$\pi_{0.5}$--RoboTwin}
Initialization is the public \texttt{pi05\_base} checkpoint. The spatial
backbone is DA3-GIANT-1.1 in backbone-only mode, frozen, with taps at layers
$19,27,33,39$ resampled to an $18\times24$ grid ($P=432$) over three
calibrated views, ordered as overhead, left wrist, right wrist. Metric depth
is ground truth from the simulator, retained at millimeter precision and
aligned to the feature grid by nearest-neighbor resampling. The two
end-effector anchors are the optical centres of the left and right wrist
views, recovered from the same extrinsics used for back-projection. The
action-token transformer has
$L=18$ blocks and the graft attaches to the final six blocks $12$--$17$.

Bank and grafting hyperparameters: residual scale $\rho=1$; output projections and FiLM initialized at standard deviation $10^{-2}$;
Fourier bands
$B=10$; workspace centre $c=(-0.0222,\,0.0967,\,0.7403)$ and isotropic scale
$\lambda=0.5198$, fitted on the training split; grid and view embeddings
scaled by $0.25$; query/key normalization enabled; attention logit gain
initialized at $3$ and clamped to $8$; whole-graft dropout
$p_{\mathrm{drop}}=0.1$; bank-only routing and bias-free auxiliary attention output and
merge projections enabled; per-tap, fused-latent and per-sample bank centering
all enabled. Evaluation predicts action chunks under the
host's own flow-matching sampler with $30$ denoising steps.

\paragraph{Fourier band count} We set $B$ against measurement precision
rather than maximum precision. At the RoboTwin workspace scale $B=10$ resolves
roughly $5$~mm in the highest band, while depth is quantized at $1$~mm and
pooled over a patch covering one to five centimeters of surface, so further
bands would encode noise rather than geometry.

\paragraph{X-VLA--RoboTwin}
The X-VLA graft uses the same bank builder and the same grafting topology
with identical settings---$\rho=1$, initialization $10^{-2}$, $B=10$,
query/key normalization, $p_{\mathrm{drop}}=0.1$, bank-only routing,
bias-free auxiliary attention output and merge projections, and per-tap, fused and
per-sample centering---grafted onto the final six of its $24$ action blocks.
The PyTorch implementation also omits the primary attention output bias;
the JAX $\pi_{0.5}$ implementation retains it. The hosts, workspace constants,
and camera layouts differ, while the bank construction and two-stage grafting
scheme are shared.

\paragraph{Fast-WAM--RoboTwin} Fast-WAM \cite{yuan2026fastwam} uses a
30-block ActionDiT with a pretrained Wan2.2 video expert. The graft starts
from the released \texttt{robotwin\_uncond\_3cam\_384} checkpoint, attaches
to blocks 20--29, and uses three calibrated views, four DA3-NESTED taps
(layers 19, 27, 33 and 39) on an $18\times24$ grid, and measured metric
depth. Only the spatial bank builder and the ten grafting modules are
trained, on the unchanged RoboTwin training split; DA3-NESTED, the video and
action experts, the variational autoencoder and the proprio encoder remain
frozen. Evaluation predicts 32 absolute-joint actions, executes 24 before
replanning, and uses 10 flow-matching steps under the raw protocol.

\paragraph{LingBot-VA--RoboTwin} LingBot-VA \cite{li2026lingbotva} is an
autoregressive video-action model with a dual-stream mixture of
transformers. The graft starts from its released RoboTwin post-training
checkpoint, retains its autoregressive video-action rollout, and attaches the
same bank-only interface to the final ten blocks of its action stream
through the linear bridges of Table~\ref{tab:param-cost}, since its blocks
are wider than the injection width; the video-generation pathway is left
unchanged.

\paragraph{Clean-minus-Randomized gap under ablation} The ungrafted
$\pi_{0.5}$ base has a $5.9\%$ Clean-minus-Randomized gap on RoboTwin; the
reference graft narrows it to $1.6\%$. Bank-off ($6.6\%$) and bank-shuffle
($5.7\%$) restore nearly the base gap, LoRA leaves it at $5.9\%$, and the
retrained variants narrow it partially: $2.4\%$ without dropout, $4.3\%$ with
all-block injection, and $2.4\%$ with VGGT-$\Omega$.

\section{Evaluation protocol and provenance}
\label{app:eval-protocol}

\paragraph{Matched comparisons}
For matched in-house base/graft runs, we hold the host checkpoint,
demonstrations, observations, action representation, training recipe, and
rollout protocol fixed; the graft supplies the spatial banks. These paired
results are comparable within a host, whereas scores from different hosts
are not controlled comparisons. Other table entries are published baselines; their provenance is specified
below.

\paragraph{Benchmark setup and training data}
On LIBERO~\cite{liu2023libero}, the simulated embodiment is a single-arm
Franka Panda. We meta-train on all 130 tasks of the original dataset, with
50 demonstrations per task (6,500 total), and evaluate over 50 seeds per task.
On RoboTwin~\cite{chen2025robotwin2}, the embodiment is a simulated dual-arm
Aloha-AgileX. We use all 50 tasks, with 50 clean and 500 randomized
trajectories per task (27,500 total), and evaluate our runs over exactly 20 seeds per task and configuration, the
same seeds for every model (Table~\ref{tab:robotwin-per-task}). Clean and Randomized share tasks but
differ in clutter, lighting, background, table height and instruction phrasing.
The VLA hosts start from broad multi-embodiment pretraining checkpoints;
the WAM hosts start from their released RoboTwin checkpoints.

For B1K~\cite{li2023behavior1k,behaviorchallenge2025}, we use the challenge's
simulated Galaxea R1Pro, a wheeled dual-arm robot, and fine-tune a separate
policy for each task from 100 demonstrations. Q-score is the fraction of
an activity's goal conditions satisfied, averaged over episodes. We evaluate on all 20 public evaluation instances of each task. Official challenge scoring uses only 10 of those 20 instances; our reported scores aggregate all 20.

On RoboPRO~\cite{li2026robopro}, we use the simulated dual-arm Aloha-AgileX
and train on the full clean and cluttered demonstration set, approximately
16,000 trajectories across 80 tasks. Evaluation reports the benchmark's
Easy and Hard success criteria. Cluttered scenes add distractors within
the demonstrated motion path while holding the underlying skills fixed.

Real-robot evaluation uses UR5e, AgileX Piper, and Galaxea R1Pro, covering
thin-object grasping, constrained placement, multi-step table cleanup, and
bimanual manipulation. Each base/graft condition is evaluated over ten
trials per task; success requires completing all subtasks.
Table~\ref{tab:real-robot} specifies the tasks and platforms.

\paragraph{Result provenance}
The main tables mix published numbers with our own runs; we record the
source of each value here rather than marking it in place. Every \method
entry is trained and evaluated in-house. \emph{LIBERO}: every non-grafted entry is quoted from published work. The
X-VLA, Fast-WAM and LingBot-VA bases are taken from their papers; the
$\pi_{0.5}$ base is taken from the openpi release \cite{openpi2025}, since
the $\pi_{0.5}$ paper reports no LIBERO result. \emph{RoboTwin}: the WAM4D, $\pi_{0.5}$ and X-VLA bases are published, the
latter two taken from \cite{li2026lingbotva} and \cite{bi2026motus}
respectively; we use these published values because those models were
trained on the same 50-task RoboTwin demonstration set as our grafts, so
base and graft are directly comparable. The Fast-WAM and LingBot-VA bases
are our evaluations of their public checkpoints with our in-house evaluation
seeds and pipeline, the same seeds used for every grafted run; and Spatial Forcing is the one reference we trained ourselves: its paper
trains on a seven-task subset, and we used its public RoboTwin training
code on the full 50-task dataset.
\emph{B1K}: the $\pi_{0.5\text{-}\mathrm{Comet}}$ and $\pi_{0.5\text{-}\mathrm{RLC}}$
scores are those published by the challenge organizers
\cite{behaviorchallenge2025} and the SERF column is quoted from
\cite{kim2026serf}; our grafted scores are evaluated on all 20 public
evaluation instances of each task, whereas official challenge scoring uses only 10 of those 20. These aggregates therefore use different instance sets. \emph{RoboPRO}: every baseline is from
\cite{li2026robopro}.

\paragraph{Published comparators}
GLaD, GeoVLA and SpatialVLA do not report RoboTwin, and WAM4D does not
report LIBERO; these NR entries reflect what the cited papers evaluate, not
what the methods can run. WAM4D's $93.8/89.9$ covers all 50 RoboTwin tasks
with 50 clean and 500 randomized demonstrations per task, but its host and
training differ from ours, so it is a level rather than a matched
comparison. On B1K, SERF reports three tasks, which we quote in
Table~\ref{tab:behavior-five}.

\section{Detailed benchmark results}
\label{app:benchmark-details}

This section reports the per-suite and per-task results behind the
aggregates of the main text. Every value in these tables that is not marked
as quoted from a published work was produced by us with the evaluation code,
seed lists and instance lists included in the supplementary material, under
the protocols of Section~\ref{app:eval-protocol}.

\paragraph{LIBERO by suite} Table~\ref{tab:libero-details} breaks the LIBERO
averages of Table~\ref{tab:libero-robotwin} into the four standard suites
(Spatial, Object, Goal and Long) for every host, ungrafted and grafted,
alongside the published reference policies. Each grafted entry is our
evaluation over 50 seeds per task with the supplied evaluation script.

\paragraph{RoboTwin by task} Table~\ref{tab:robotwin-per-task} lists
Clean/Randomized success on each of the 50 RoboTwin tasks for every policy we
evaluated ourselves: the four grafted hosts, the ungrafted Fast-WAM and
LingBot-VA checkpoints, and the VGGT-$\Omega$ variant of grafted
$\pi_{0.5}$. All runs use the same 20 seeds per task; the seed bank and the
evaluation pipeline that produced these numbers are included in the
supplementary material.

\paragraph{RoboPRO by task} Table~\ref{tab:robopro-per-task} lists Easy/Hard
success for the two grafted hosts on each of the 80 RoboPRO tasks, in clean
and cluttered scenes, over 20 seeds per task. The evaluation configuration
and seed lists are included in the supplementary material.

\section{Limitations}
\label{sec:limitations}

Three limitations remain. First, we evaluated only flow-matching action
experts; other policies with accessible action tokens and residual updates
could use the interface but remain untested. Second, the length of the graft-only warmup depended on how close each host
already was to the target data; no single setting was best for all four
hosts. Third, errors in
the frozen backbone's depth or features limit the available geometry,
particularly for high-precision contact.

\section{Conclusion}

This paper started from a specific failure of pretrained flow-matching
policies: VLAs and WAMs select the right object and the right sequence of
motions, yet misplace the grasp, because their 2D backbones encode neither
metric scale nor where a surface lies relative to the gripper. We argued
that the missing ingredient is not geometry as such, which sensors and
reconstruction models already supply, but an interface. Prior spatial action
models either consume reconstruction latents without metric grounding or
supply explicit metric geometry without the latents; each is built for one
policy family, validated in a limited number of settings, and often dependent on inputs a
deployed robot lacks. \method is that interface. It binds each frozen
reconstruction latent, at its own grid location, to its absolute metric
position and its offset from every end effector, and grafts the bound tokens
by bank-only cross-attention into the final blocks of the host's action
expert, from RGB, calibration and depth alone.

The experiments answer the three questions this framing poses. One graft
improves every host: the same construction, unchanged, improves two VLAs and
two WAMs initialized from public checkpoints on RoboTwin by $+7.5\%$ Clean and
$+8.2\%$ Randomized on average, with grafted $\pi_{0.5}$ reaching
$94.0/92.4$, above the strongest published 3D-conditioned policy, while
leaving saturated LIBERO intact. The gains hold where geometry matters most:
all eight RoboPRO host--condition cells improve, including the stricter
collision-penalizing criterion under clutter; on B1K the graft improves the
challenge-winning policy's performance on five of six tasks, by up to $0.47$ Q-score and
most on the tasks that fail at the grasp, and exceeds a map-conditioned
policy in the mean of the tasks both report without keeping any scene
state; and on three real robots every task improves by $20\%$ to $50\%$.
The gain comes from the grounded geometry and its placement rather than from
added parameters: zeroing the banks drops grafted $\pi_{0.5}$ below its own
base, substituting another episode's banks drops it a further $24\%$,
injecting into every block is worse than injecting late, and swapping DA3
for VGGT-$\Omega$ keeps most of the gain. Together these results support
the claim of the introduction: a pretrained action expert can act on
reconstruction latents once they are expressed in absolute metric coordinates and end-effector-relative
offsets, and one such interface serves the current generation of
flow-matching policies. Section~\ref{sec:limitations} records what remains
open.

\bibliographystyle{IEEEtran}
\bibliography{references}

\clearpage
% Extended qualitative figure and detailed tables follow the references.
\begin{figure*}[p]
\centering
\begingroup
\setlength{\tabcolsep}{3pt}
\setlength{\fboxsep}{0pt}
\setlength{\fboxrule}{0.35pt}
\renewcommand{\arraystretch}{1.02}
\newcommand{\resultframe}[1]{\fcolorbox{black!45}{white}{\includegraphics[width=1.72in]{figures/qualitative_results/#1.png}}}
\begin{tabular}{@{}lcc@{}}
 & \scriptsize Successful rollout & \scriptsize Failure case \\
\scriptsize \shortstack[l]{Real R1Pro\\fruit placement\\\\The model \\without grafting\\ leaves insufficient\\ clearance when placing\\ the tangerine.}
  & \resultframe{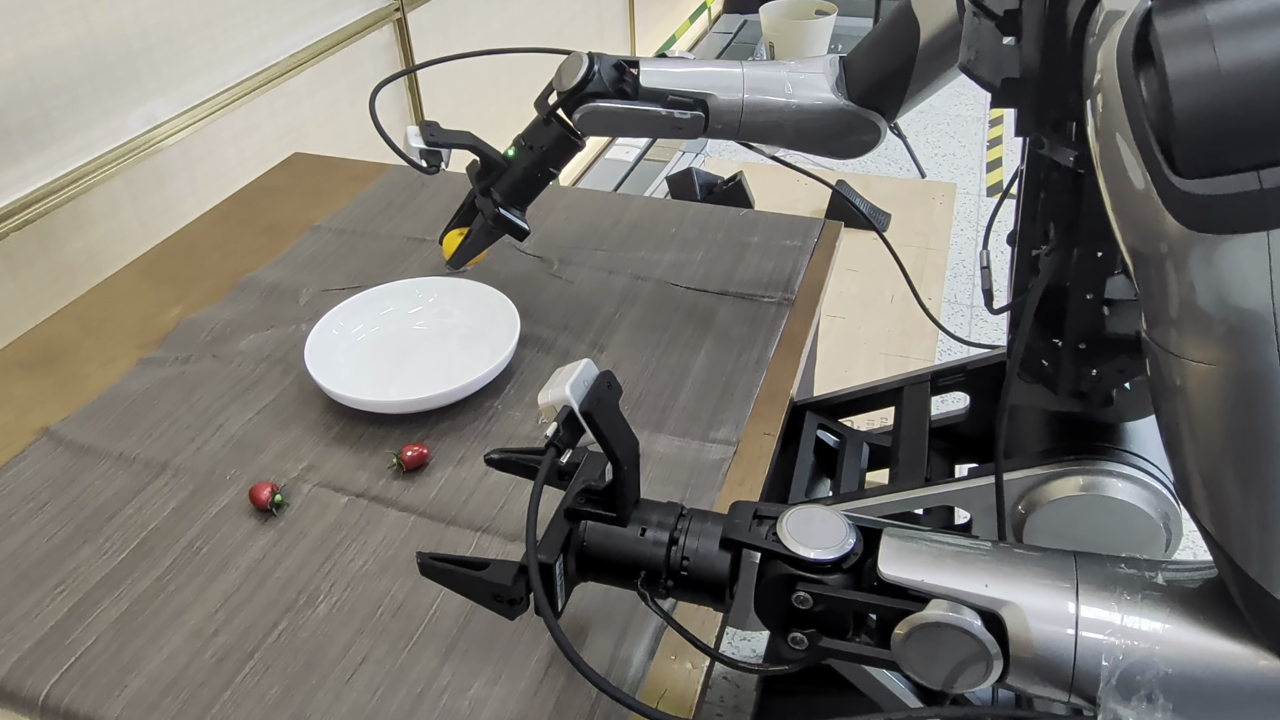} & \resultframe{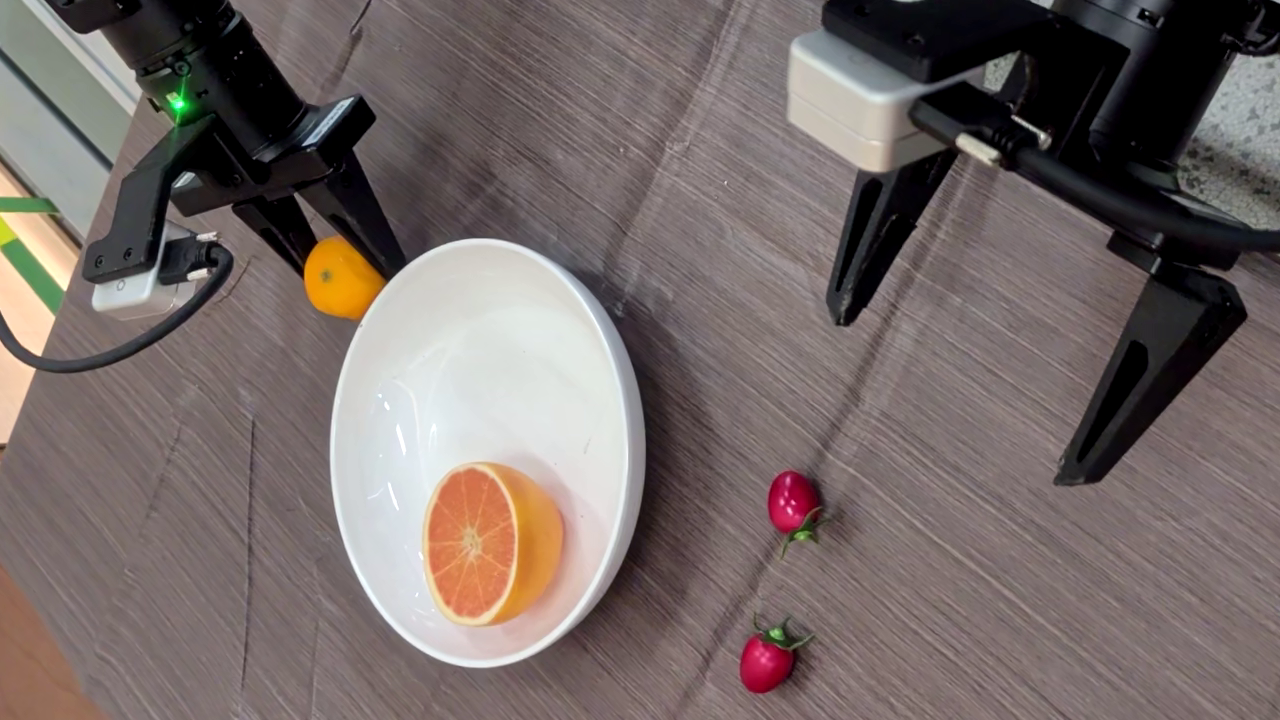} \\
\scriptsize \shortstack[l]{Real Piper\\plate placement\\(wrist views)\\\\The model \\without grafting \\ misaligns the plate \\and collides with the rack. }
  & \resultframe{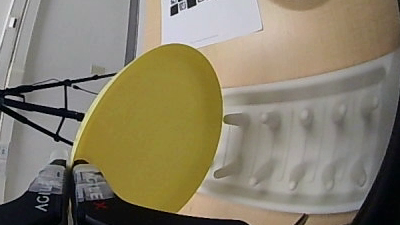} & \resultframe{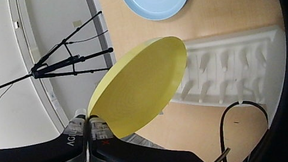} \\
\scriptsize \shortstack[l]{Real Piper\\tea-can opening\\\\The model\\without grafting\\grasps the tea-can lid \\too high\\ and immediately drops it.}
  & \resultframe{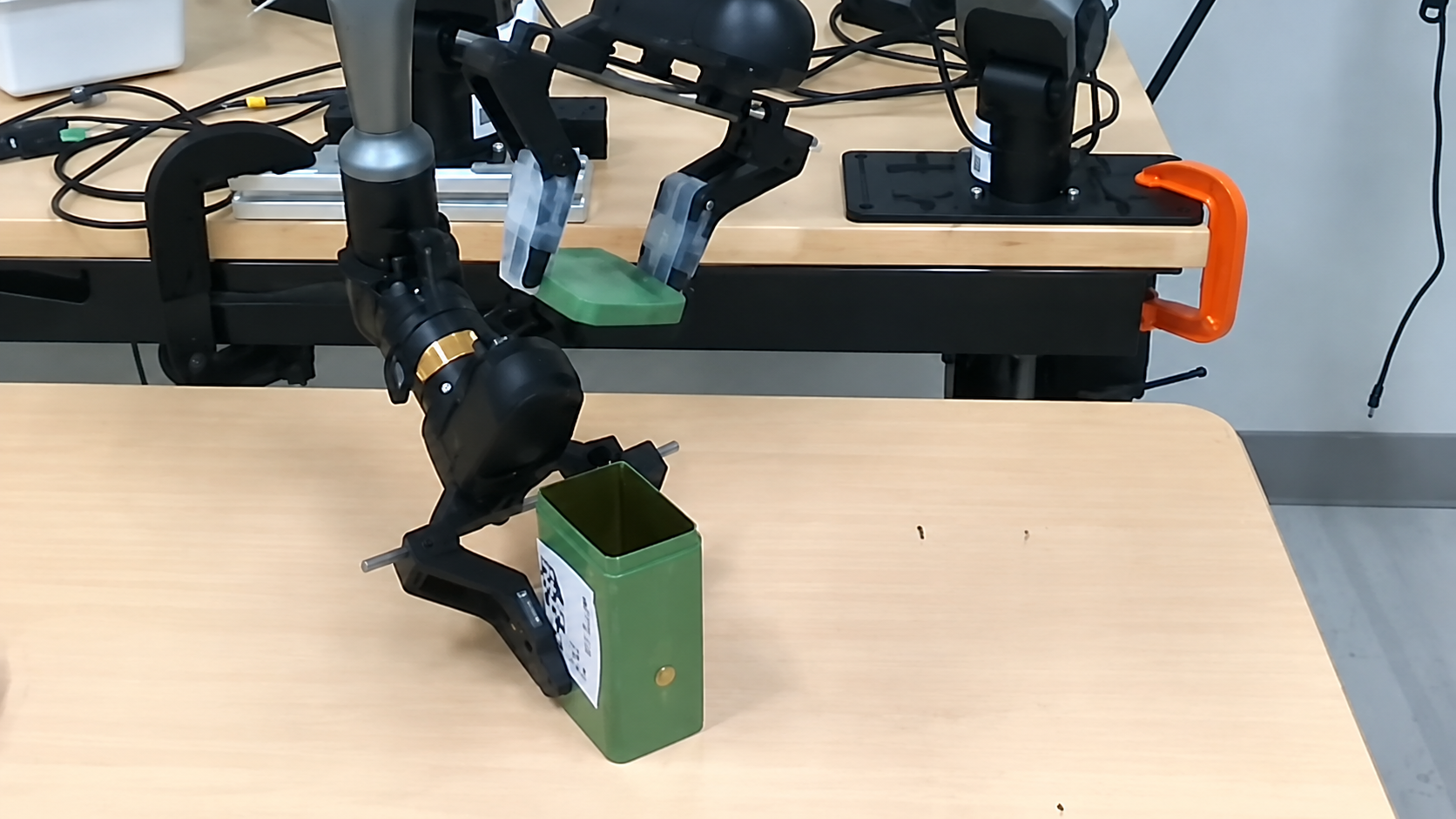} & \resultframe{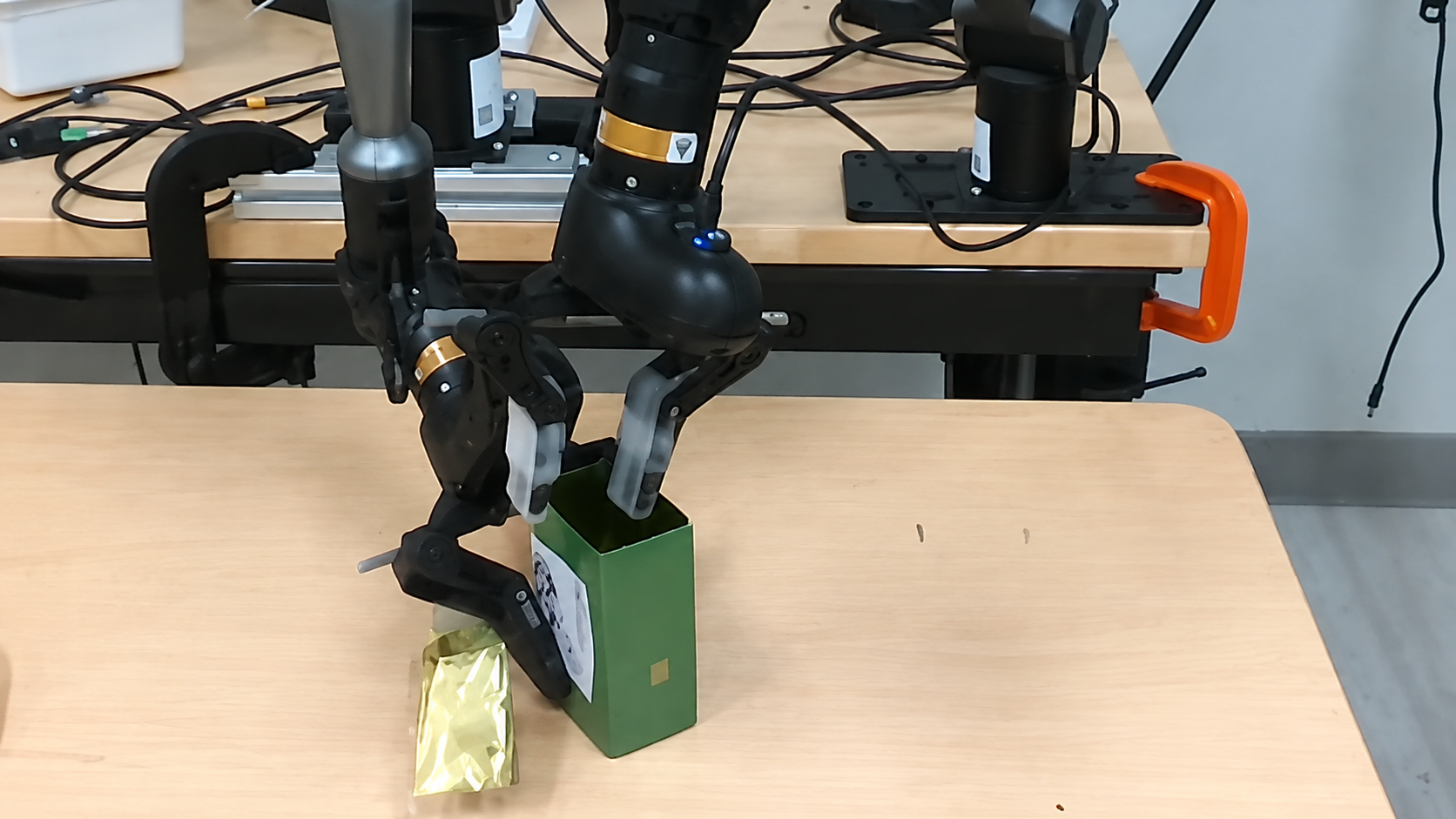} \\
\scriptsize \shortstack[l]{Real UR5e\\marker grasp\\\\The model\\without grafting\\ grasps the marker\\too high.}
  & \resultframe{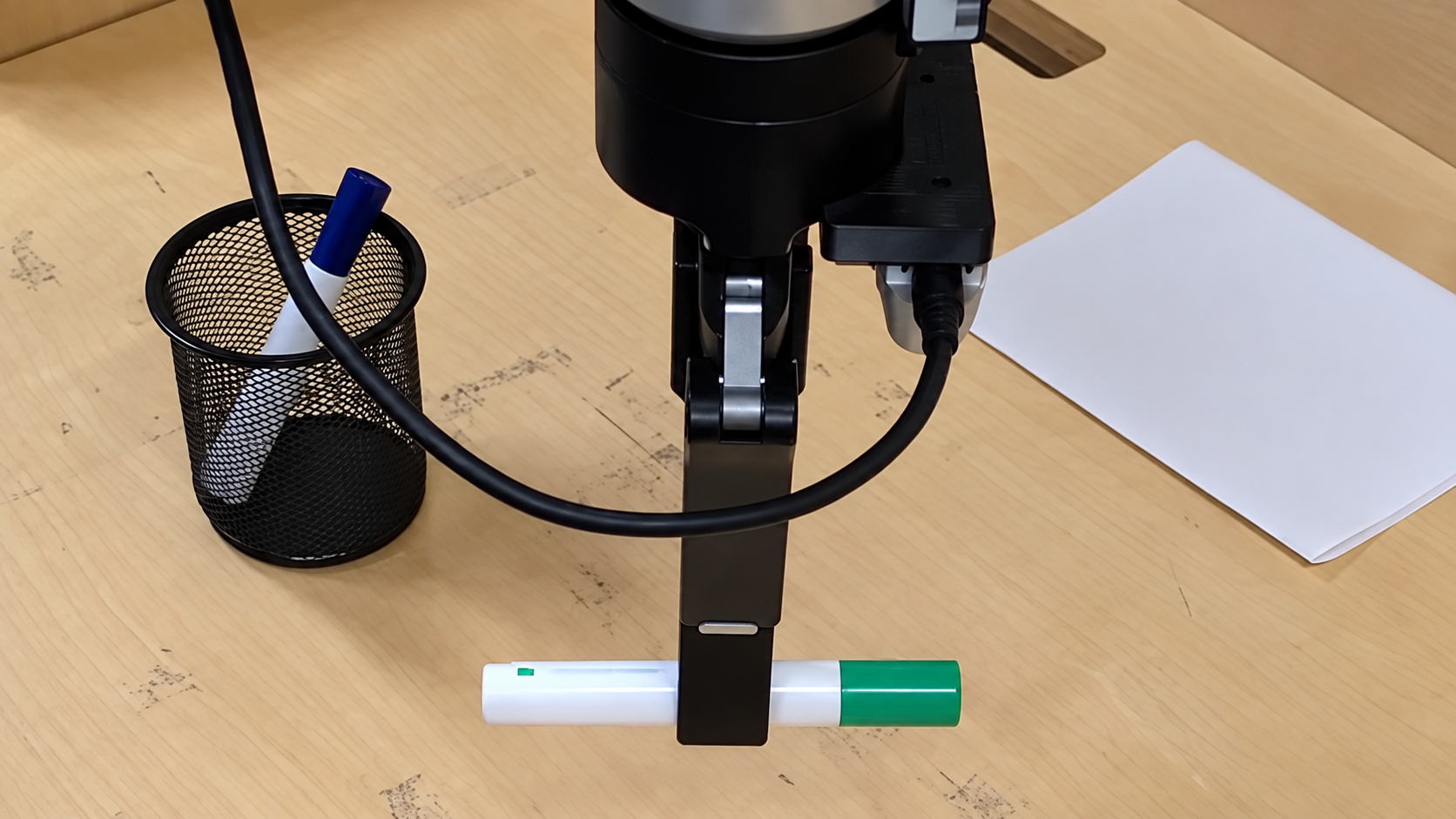} & \resultframe{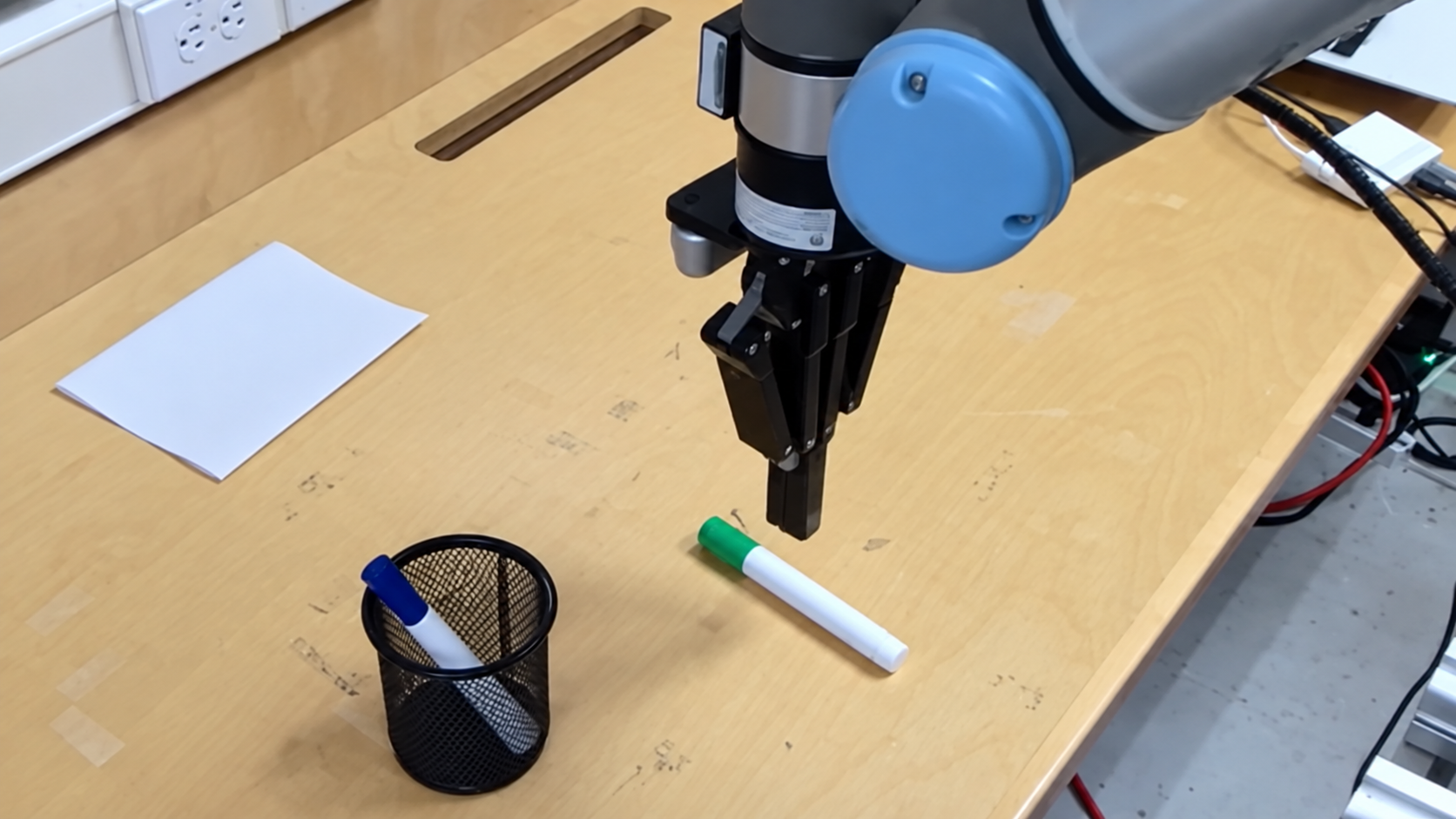} \\
\scriptsize \shortstack[l]{B1K\\(R1Pro), candy pickup\\\\The model\\without grafting\\ grasps the candy\\too high.}
  & \resultframe{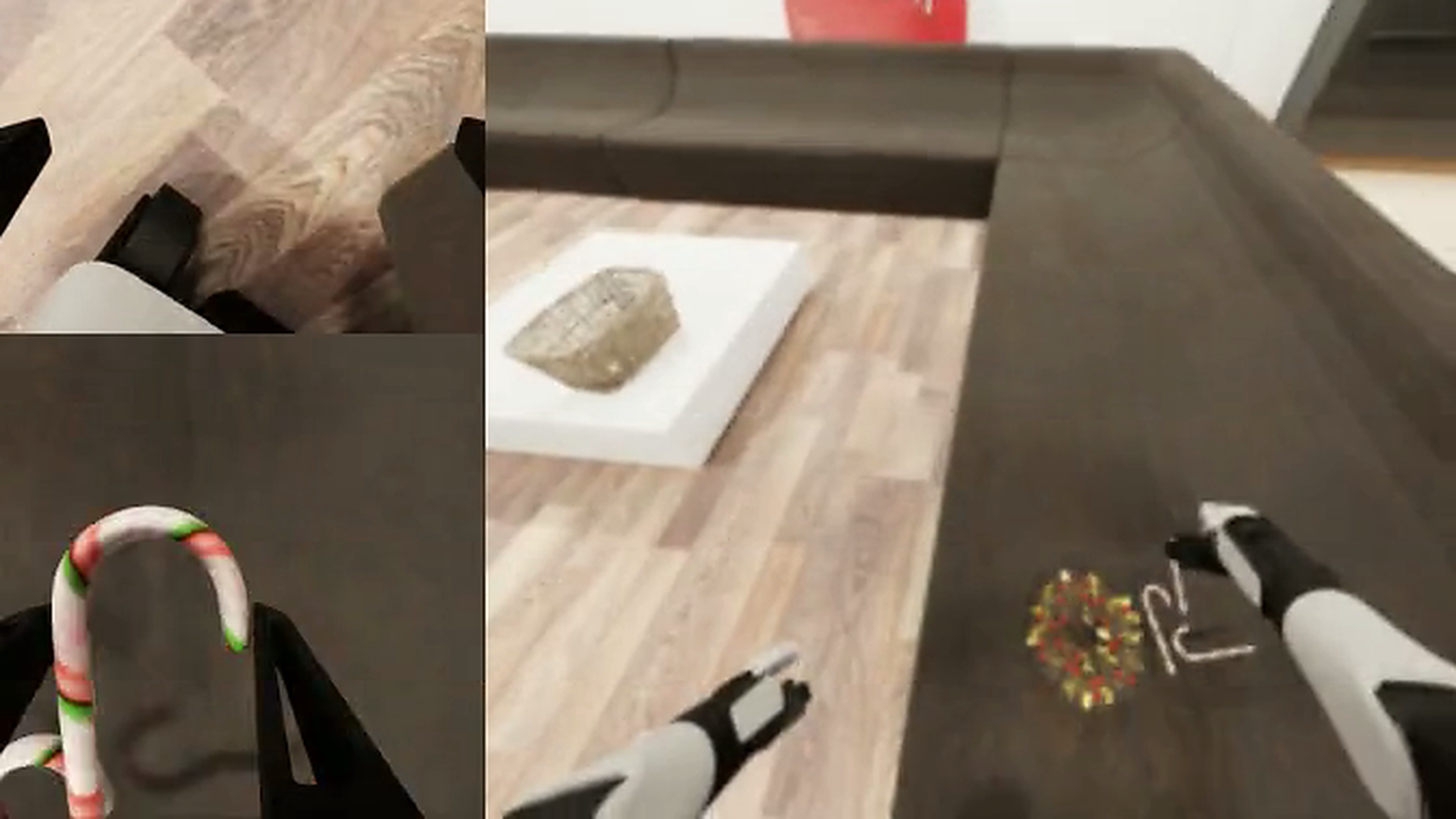} & \resultframe{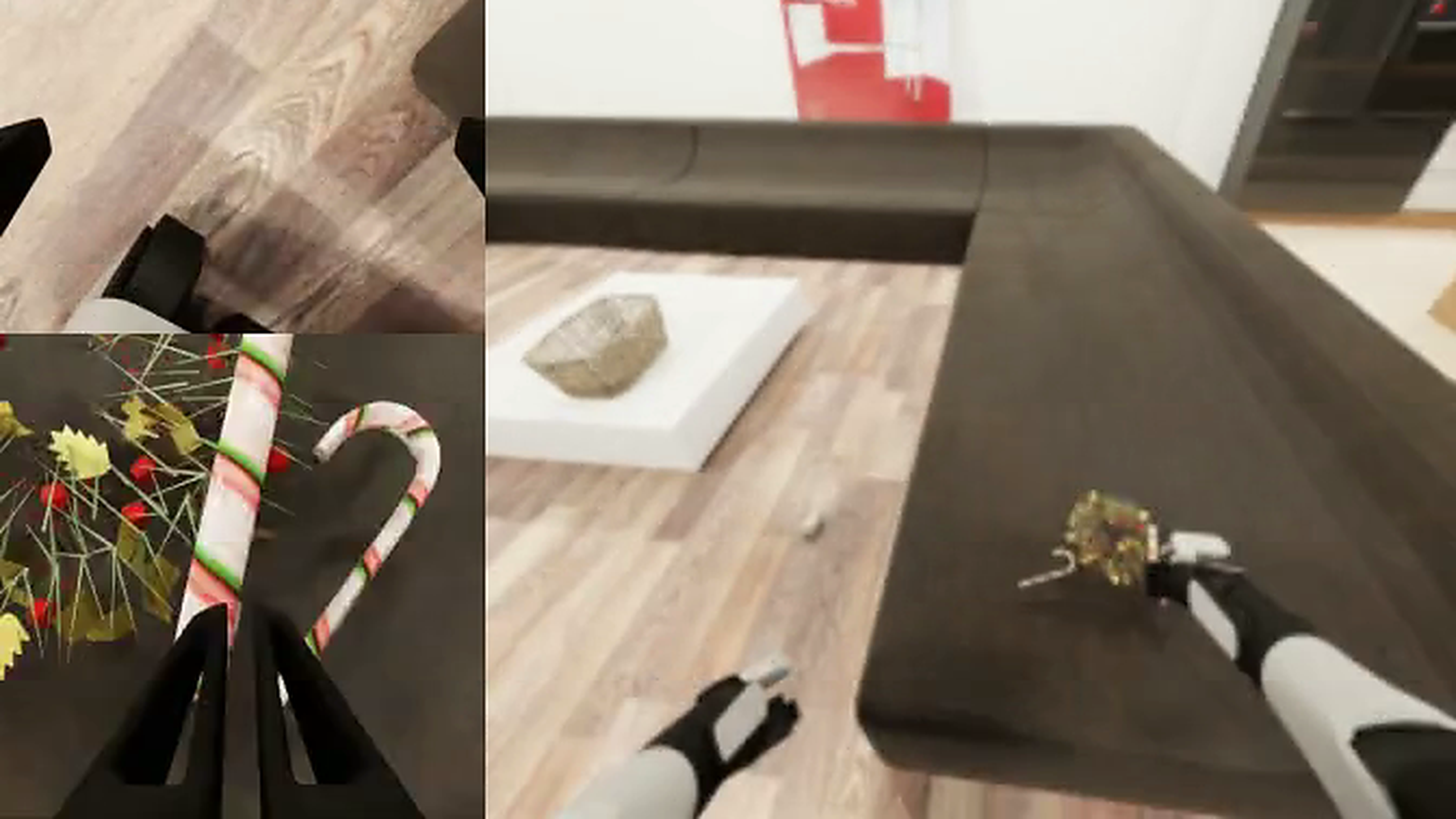} \\
\scriptsize \shortstack[l]{RoboTwin\\(ALOHA), hang mug\\\\The model\\without grafting\\provides insufficient\\clearance to hang the mug.}
  & \resultframe{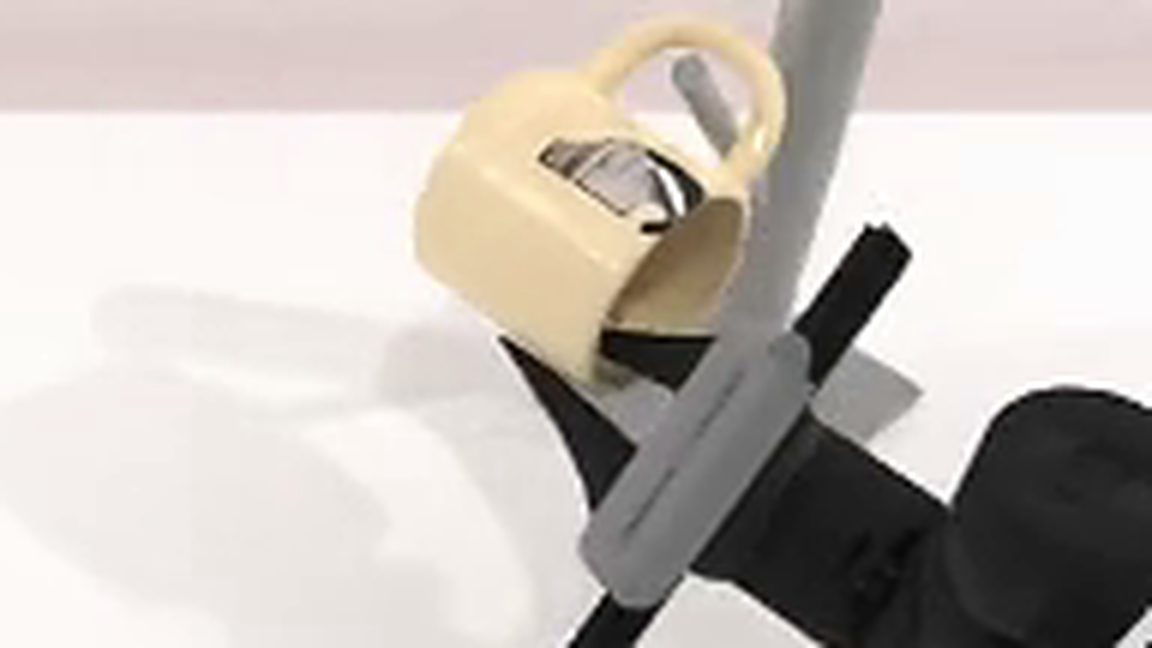} & \resultframe{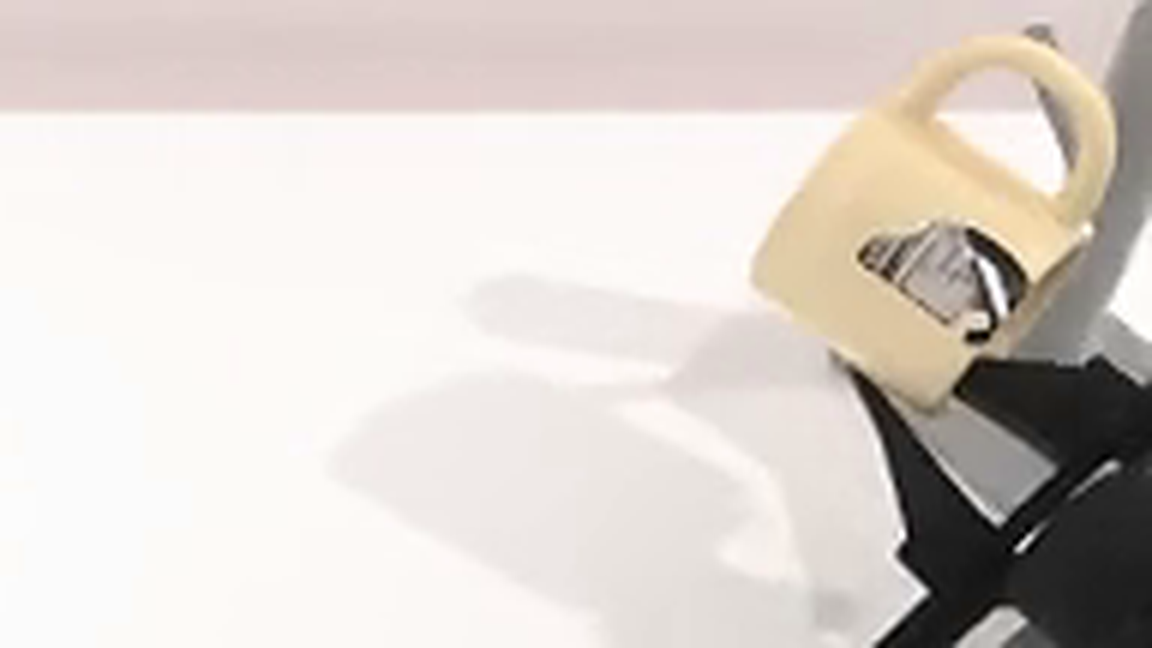} \\
\scriptsize \shortstack[l]{RoboPRO\\(ALOHA), plate pickup\\\\The model\\without grafting\\ grasps the plate\\too high.}
  & \resultframe{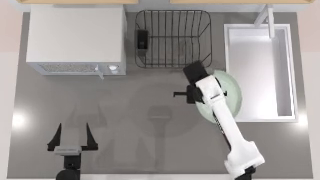} & \resultframe{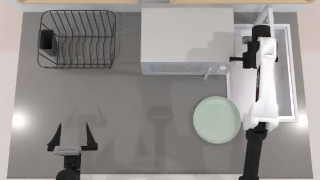}
\end{tabular}
\endgroup
\caption{\textbf{Success and failure modes across platforms.} Left and right
show successful and failed rollouts. For R1Pro fruit placement, improved
plate-height estimation leaves clearance above the rim, while the failure
approaches too low. For Piper plate placement, better plate--rack alignment
avoids collision; the two panels use left- and right-wrist fisheye cameras,
respectively. In B1K candy pickup and RoboPRO plate pickup, target
depth supports a lower grasp pose. In RoboTwin mug hanging, estimating the
rack height positions the mug at the support. The added Piper tea-can
opening and UR5e marker-grasp panels are cropped and AI-enhanced for
visual clarity; generative deblurring may reconstruct fine details, so
these panels are illustrative rather than unaltered experimental frames.}
\label{fig:qualitative-results-full}
\end{figure*}
\clearpage
\begin{table*}[p]
\centering
\scriptsize
\setlength{\tabcolsep}{3pt}
\renewcommand{\arraystretch}{1.05}
\caption{\textbf{Host policies and geometry-aware policies.} Rows follow
Section~\ref{sec:related}. Status is the venue as of submission; ``Real''
denotes real-robot experiments; test-time inputs are what the deployed policy
consumes beyond proprioception. The four hosts carry no geometry module.
Every geometry-aware method reports at most three evaluation settings;
\method reports five, on four hosts.}
\label{tab:spatial_vla_comparison}

\begin{tabularx}{\textwidth}{@{}
>{\raggedright\arraybackslash}p{0.10\textwidth}
>{\raggedright\arraybackslash}p{0.075\textwidth}
>{\raggedright\arraybackslash}p{0.17\textwidth}
>{\raggedright\arraybackslash}X
>{\raggedright\arraybackslash}p{0.13\textwidth}
>{\raggedright\arraybackslash}p{0.17\textwidth}@{}}
\toprule
\textbf{Method} & \textbf{Status} & \textbf{Geometry signal} &
\textbf{Route into the model} & \textbf{Test-time inputs} & \textbf{Evaluation} \\
\midrule

\multicolumn{6}{@{}l}{\emph{Hosts: VLAs and WAMs}} \\
\addlinespace[2pt]
$\pi_{0.5}$ & CoRL'25 & None &
VLM trunk co-trained with a flow-matching action expert & RGB & Real \\
X-VLA & ICLR'26 & None &
Embodiment-specific soft prompts into a flow-matching transformer & RGB &
LIBERO; SimplerEnv; CALVIN; VLABench; RoboTwin; Real \\
Fast-WAM & arXiv'26 & None &
Video and action experts in a shared mixture of transformers; no video
generation at inference & RGB & LIBERO; RoboTwin; Real \\
LingBot-VA & RSS'26 & None &
Autoregressive video/action tokens in a mixture of transformers & RGB &
LIBERO; RoboTwin; Real \\

\midrule
\multicolumn{6}{@{}l}{\emph{Spatial action models compared in Section~\ref{sec:experiments}}} \\
\addlinespace[2pt]
SpatialVLA & RSS'25 & Egocentric 3D points from estimated depth &
Position encoding added to the VLM's visual tokens & RGB (depth estimated) &
LIBERO; SimplerEnv; Real \\
GeoVLA & IROS'26 & Point cloud from depth &
Learned point encoder; tokens fed with VLM embeddings to a 3D-enhanced
action expert & RGB-D sensor & LIBERO; ManiSkill2; Real \\
Spatial Forcing & ICLR'26 & Frozen VGGT features (training only) &
Alignment of intermediate visual embeddings; no 3D input at inference & RGB &
LIBERO; RoboTwin; Real \\
GLaD & arXiv'25 & Frozen VGGT features (training only) &
Distillation into the language model's hidden states over visual tokens & RGB &
LIBERO; LIBERO-PRO \\
WAM4D & arXiv'26 & Future depth (training only) &
Spatial register tokens and a depth head, removed for deployment & RGB &
RoboTwin; Real \\
SERF & CoRL'26 & Persistent neural-point map of scene and robot &
Map tokens prepended to $\pi_{0.5}$'s inputs & RGB-D, simulator instance labels, prior map & B1K (3 tasks) \\

\midrule
\multicolumn{6}{@{}l}{\emph{Wider geometry-aware literature}} \\
\addlinespace[2pt]
PointVLA & RA-L'26 & Point-cloud features &
Additive adapters in the action expert & RGB + point cloud & RoboTwin~1.0; Real \\
GeoPredict & CVPR'26 & 3D keypoint tracks and Gaussian geometry (training only) &
Auxiliary predictive supervision & RGB & LIBERO; RoboCasa; Real \\
3D-Mix & arXiv'26 & Frozen VGGT features &
Gated fusion with the VLM's semantic features & RGB & LIBERO; SimplerEnv \\
DyWA & ICCV'25 & Partial and goal point clouds &
Joint action and next-state prediction & Point cloud & Custom simulation; Real \\
GWM & ICCV'25 & Dynamic 3D Gaussians &
Action-conditioned Gaussian rollouts feed a separate policy & RGB &
Meta-World; RoboCasa; Real \\
X-WAM & arXiv'26 & Future depth &
Replicated late diffusion-transformer blocks predict depth & RGB &
RoboCasa; RoboTwin; Real \\
\mbox{MECo-WAM} & arXiv'26 & Frozen VGGT 4D targets (training only) &
4D expert with decayed read-mask attention, removed for deployment & RGB &
LIBERO; RoboTwin; Real \\

\midrule
\textbf{Ours} & --- &
\textbf{Frozen multi-layer features + metric geometry} &
\textbf{Bank-only cross-attention into the late action layers} &
\textbf{RGB, calibration, depth (measured or predicted)} &
\textbf{LIBERO; RoboTwin; RoboPRO; B1K; Real} \\

\bottomrule
\end{tabularx}
\end{table*}

\begin{table*}[p]
\centering
\caption{\textbf{Detailed LIBERO SR (\%).} Results across the four
standard suites. Values for external methods are reported by their respective works; the
$\pi_{0.5}$ base is from the openpi release \cite{openpi2025}. Host averages match Table~\ref{tab:libero-robotwin}; dashes indicate
unavailable per-suite results, not zero success. Blocks group other baselines, existing 3D methods, and our hosts. Green/red:
change from the host's own base.}
\label{tab:libero-details}
\scriptsize
\setlength{\tabcolsep}{0pt}
\settowidth{\gw}{\textbf{99.0}\dpos{6.6}}
\begin{tabular*}{\textwidth}{@{\extracolsep{\fill}}l B@{\extracolsep{0pt}\hspace{0.6pt}$\rightarrow$\hspace{0.6pt}}G@{\extracolsep{\fill}} B@{\extracolsep{0pt}\hspace{0.6pt}$\rightarrow$\hspace{0.6pt}}G@{\extracolsep{\fill}} B@{\extracolsep{0pt}\hspace{0.6pt}$\rightarrow$\hspace{0.6pt}}G@{\extracolsep{\fill}} B@{\extracolsep{0pt}\hspace{0.6pt}$\rightarrow$\hspace{0.6pt}}G@{\extracolsep{\fill}} B@{\extracolsep{0pt}\hspace{0.6pt}$\rightarrow$\hspace{0.6pt}}G@{}}
\toprule
Method & \multicolumn{2}{c}{Spat.} & \multicolumn{2}{c}{Obj.} & \multicolumn{2}{c}{Goal}
& \multicolumn{2}{c}{Long} & \multicolumn{2}{c}{Avg.} \\
\midrule
$\pi_0$ & \multicolumn{2}{c}{96.8} & \multicolumn{2}{c}{98.8} & \multicolumn{2}{c}{95.8} & \multicolumn{2}{c}{85.2} & \multicolumn{2}{c}{94.1} \\
OpenVLA-OFT & \multicolumn{2}{c}{97.6} & \multicolumn{2}{c}{98.4} & \multicolumn{2}{c}{97.9} & \multicolumn{2}{c}{94.5} & \multicolumn{2}{c}{97.1} \\
\midrule
GLaD & \multicolumn{2}{c}{95.0} & \multicolumn{2}{c}{97.4} & \multicolumn{2}{c}{94.4} & \multicolumn{2}{c}{89.4} & \multicolumn{2}{c}{94.1} \\
GeoVLA & \multicolumn{2}{c}{98.4} & \multicolumn{2}{c}{99.0} & \multicolumn{2}{c}{96.6} & \multicolumn{2}{c}{96.6} & \multicolumn{2}{c}{97.7} \\
Spatial Forcing & \multicolumn{2}{c}{\underline{99.4}} & \multicolumn{2}{c}{99.6} & \multicolumn{2}{c}{\underline{98.8}} & \multicolumn{2}{c}{96.0} & \multicolumn{2}{c}{98.5} \\
SpatialVLA & \multicolumn{2}{c}{88.2} & \multicolumn{2}{c}{89.9} & \multicolumn{2}{c}{78.6} & \multicolumn{2}{c}{55.5} & \multicolumn{2}{c}{78.1} \\
\midrule
$\pi_{0.5}$ & 98.8 & \textbf{99.8}\dpos{1.0} & 98.2 & \underline{99.8}\dpos{1.6} & 98.0 & \textbf{99.8}\dpos{1.8} & 92.4 & \textbf{99.0}\dpos{6.6} & 96.9 & \textbf{99.6}\dpos{2.7} \\
X-VLA & 98.2 & \underline{99.4}\dpos{1.2} & 98.6 & 99.6\dpos{1.0} & 97.8 & \underline{98.8}\dpos{1.0} & 97.6 & 97.0\dneg{0.6} & 98.1 & \underline{98.7}\dpos{0.6} \\
Fast-WAM & 98.2 & 98.0\dneg{0.2} & \textbf{100.0} & 98.4\dneg{1.6} & 97.0 & 97.2\dpos{0.2} & 95.2 & 94.0\dneg{1.2} & 97.6 & 96.9\dneg{0.7} \\
LingBot-VA & 98.5 & 98.6\dpos{0.1} & 99.6 & 99.0\dneg{0.6} & 97.2 & 98.2\dpos{1.0} & \underline{98.5} & 93.8\dneg{4.7} & 98.5 & 97.4\dneg{1.1} \\
\bottomrule
\multicolumn{11}{@{}l}{\emph{\method: base $\rightarrow$ graft}} \\
\end{tabular*}
\end{table*}

\begin{table*}[p]
\centering
\scriptsize
\setlength{\tabcolsep}{2pt}
\renewcommand{\arraystretch}{1.0}
\caption{\textbf{Per-task RoboTwin SR (\%), Clean/Randomized.} Each cell
reads Clean/Randomized over exactly 20 seeds per task; every model and
evaluation uses the same seeds. DA3 is the spatial backbone unless marked
VGGT-$\Omega$; Fast-WAM and LingBot-VA without a superscript are the
ungrafted public checkpoints. The average is the mean of the 50 per-task rates.}
\label{tab:robotwin-per-task}
\begin{tabular*}{\linewidth}{@{\extracolsep{\fill}}lccccccc@{}}
\toprule
Task & $\pi_{0.5}^{\mathrm{graft}}$ & X-VLA$^{\mathrm{graft}}$ & Fast-WAM & \shortstack{Fast-\\WAM$^{\mathrm{graft}}$} & LingBot-VA & \shortstack{LingBot-\\VA$^{\mathrm{graft}}$}
& \shortstack{$\pi_{0.5}^{\mathrm{graft}}$\\(VGGT-$\Omega$)} \\
\midrule
Adjust bottle & 100/100 & 100/95 & 100/95 & 100/100 & 100/100 & 95/95 & 100/100 \\
Beat block hammer & 100/100 & 90/85 & 100/100 & 100/95 & 100/100 & 100/100 & 100/95 \\
Blocks ranking RGB & 100/95 & 100/85 & 95/85 & 100/90 & 75/90 & 100/100 & 100/100 \\
Blocks ranking size & 85/70 & 75/85 & 60/65 & 65/85 & 85/50 & 95/90 & 85/75 \\
Click alarmclock & 95/100 & 10/30 & 100/100 & 100/100 & 100/100 & 100/100 & 95/100 \\
Click bell & 100/100 & 70/90 & 100/100 & 100/100 & 100/100 & 100/100 & 95/100 \\
Dump bin bigbin & 95/90 & 95/90 & 95/95 & 95/95 & 90/100 & 85/95 & 80/90 \\
Grab roller & 100/100 & 100/100 & 100/100 & 100/100 & 100/100 & 100/100 & 100/100 \\
Handover block & 90/95 & 100/85 & 60/95 & 95/85 & 80/80 & 100/95 & 95/95 \\
Handover mic & 100/100 & 95/90 & 95/85 & 100/100 & 65/100 & 95/95 & 100/95 \\
Hanging mug & 50/45 & 55/45 & 15/25 & 40/35 & 35/30 & 30/30 & 20/15 \\
Lift pot & 100/100 & 100/100 & 95/95 & 100/90 & 100/100 & 100/100 & 95/90 \\
Move can pot & 100/95 & 100/85 & 80/75 & 95/70 & 85/100 & 90/100 & 100/95 \\
Move pillbottle pad & 100/100 & 95/100 & 100/100 & 100/95 & 95/95 & 100/100 & 100/100 \\
Move playingcard away & 100/100 & 100/100 & 95/100 & 100/100 & 100/95 & 100/100 & 100/100 \\
Move stapler pad & 90/85 & 70/75 & 65/70 & 65/75 & 75/70 & 55/65 & 80/70 \\
Open laptop & 100/100 & 90/95 & 95/100 & 90/100 & 100/100 & 100/95 & 95/95 \\
Open microwave & 70/70 & 55/10 & 55/75 & 80/50 & 90/75 & 85/80 & 75/80 \\
Pick diverse bottles & 95/100 & 50/10 & 85/95 & 100/100 & 80/90 & 90/90 & 95/80 \\
Pick dual bottles & 100/100 & 100/100 & 95/95 & 100/95 & 100/100 & 100/100 & 100/100 \\
Place a2b left & 95/100 & 85/80 & 75/35 & 45/55 & 100/75 & 100/100 & 100/100 \\
Place a2b right & 95/100 & 70/75 & 70/30 & 40/55 & 95/60 & 100/95 & 100/100 \\
Place bread basket & 100/100 & 95/90 & 90/85 & 95/100 & 100/95 & 100/100 & 95/90 \\
Place bread skillet & 80/85 & 90/85 & 85/85 & 95/80 & 100/90 & 90/90 & 80/95 \\
Place burger fries & 95/95 & 100/95 & 95/100 & 90/95 & 90/95 & 100/85 & 100/95 \\
Place can basket & 85/75 & 80/75 & 70/85 & 80/70 & 80/90 & 85/80 & 90/75 \\
Place cans plasticbox & 100/100 & 100/100 & 100/100 & 100/100 & 100/100 & 100/100 & 100/100 \\
Place container plate & 100/100 & 100/100 & 95/90 & 100/100 & 100/100 & 100/95 & 100/100 \\
Place dual shoes & 100/80 & 80/80 & 90/95 & 85/80 & 65/90 & 100/90 & 90/80 \\
Place empty cup & 100/100 & 80/90 & 70/75 & 95/55 & 100/65 & 100/100 & 100/100 \\
Place fan & 100/95 & 75/90 & 90/100 & 80/100 & 90/95 & 100/95 & 100/95 \\
Place mouse pad & 80/95 & 70/75 & 75/80 & 75/80 & 75/75 & 100/85 & 80/70 \\
Place object basket & 85/95 & 90/95 & 55/40 & 50/55 & 85/50 & 90/90 & 85/70 \\
Place object scale & 100/95 & 85/90 & 90/40 & 60/75 & 90/90 & 95/95 & 90/100 \\
Place object stand & 100/100 & 90/100 & 90/75 & 90/90 & 100/85 & 100/95 & 100/100 \\
Place phone stand & 100/90 & 95/85 & 90/90 & 95/90 & 100/90 & 100/95 & 90/90 \\
Place shoe & 100/100 & 100/100 & 85/95 & 95/100 & 90/95 & 100/100 & 95/95 \\
Press stapler & 90/95 & 100/95 & 95/95 & 95/95 & 95/90 & 80/90 & 85/95 \\
Put bottles dustbin & 100/85 & 95/95 & 80/95 & 95/85 & 85/100 & 90/75 & 100/90 \\
Put object cabinet & 80/75 & 75/60 & 60/0 & 55/25 & 85/60 & 85/80 & 95/70 \\
Rotate QR code & 90/100 & 100/90 & 95/90 & 90/90 & 85/100 & 95/90 & 100/90 \\
Scan object & 80/85 & 90/80 & 5/10 & 40/0 & 70/5 & 100/95 & 90/80 \\
Shake bottle & 100/100 & 100/100 & 100/100 & 100/100 & 100/100 & 100/100 & 100/100 \\
Shake bottle horizontally & 100/100 & 100/100 & 100/100 & 100/100 & 100/100 & 100/100 & 100/100 \\
Stack blocks three & 95/90 & 25/35 & 100/100 & 80/100 & 90/100 & 100/80 & 55/75 \\
Stack blocks two & 100/85 & 95/85 & 100/100 & 100/95 & 100/100 & 100/100 & 95/95 \\
Stack bowls three & 95/90 & 65/75 & 70/95 & 85/85 & 60/85 & 90/90 & 80/70 \\
Stack bowls two & 100/100 & 85/95 & 95/100 & 95/95 & 100/95 & 100/100 & 100/100 \\
Stamp seal & 100/85 & 65/70 & 80/95 & 100/80 & 80/80 & 90/90 & 90/85 \\
Turn switch & 85/80 & 55/90 & 50/60 & 70/75 & 80/55 & 55/60 & 80/75 \\
\midrule
Average & 94.0/92.4 & 83.7/82.6 & 82.6/81.8 & 86.0/83.3 & 88.9/85.7 & 93.3/91.4 & 91.5/89.1 \\
\bottomrule
\end{tabular*}
\end{table*}

\begin{table*}[p]
\centering
\scriptsize
\renewcommand{\arraystretch}{0.74}
\setlength{\tabcolsep}{9pt}
\caption{\textbf{Per-task RoboPRO success (\%), Easy/Hard.} Each cell reads
Easy (SR)/Hard (HSR) over exactly 20 seeds per task for all 80 tasks in each
configuration, the same seeds for every model; averages are the mean of the 80
per-task rates and match
Table~\ref{tab:robopro}.}
\label{tab:robopro-per-task}
\begin{tabular}{@{}lcccc@{}}
\toprule
 & \multicolumn{2}{c}{$\pi_{0.5}^{\mathrm{graft}}$} & \multicolumn{2}{c}{X-VLA$^{\mathrm{graft}}$} \\
\cmidrule(lr){2-3}\cmidrule(l){4-5}
Task & Clean & Clutter & Clean & Clutter \\
\midrule
Chain apple bin bowl rack spoon sink & 15/15 & 10/10 & 0/0 & 0/0 \\
Chain apple sink plate bread board & 95/85 & 100/90 & 0/0 & 0/0 \\
Chain bowl rack apple sink & 0/0 & 0/0 & 0/0 & 0/0 \\
Chain heat hamburger & 45/45 & 85/75 & 20/20 & 0/0 \\
Chain serve hamburger & 45/45 & 0/0 & 0/0 & 0/0 \\
Close drawer & 100/100 & 100/100 & 0/0 & 0/0 \\
Close microwave & 100/100 & 100/40 & 60/60 & 35/35 \\
Drop apple in bin & 100/95 & 80/80 & 80/20 & 55/0 \\
Empty box & 95/95 & 100/95 & 80/80 & 30/30 \\
Move book onto table & 95/85 & 100/70 & 100/100 & 95/95 \\
Move bottle & 45/45 & 65/55 & 40/40 & 75/70 \\
Move bottle from fridge next to can & 100/100 & 80/80 & 60/0 & 55/0 \\
Move can from cabinet to basket & 60/60 & 65/60 & 80/80 & 5/5 \\
Move cup & 70/70 & 35/25 & 20/20 & 40/40 \\
Move cup next to book & 100/100 & 80/80 & 100/100 & 45/15 \\
Move cup onto table & 95/95 & 60/45 & 60/60 & 75/75 \\
Move cup put pen in cup & 70/70 & 35/10 & 100/100 & 20/15 \\
Move cups into box & 95/95 & 70/40 & 60/60 & 20/20 \\
Move hamburger onto plate & 100/100 & 90/90 & 0/0 & 40/40 \\
Move items around & 20/20 & 0/0 & 0/0 & 0/0 \\
Move milk close fridge & 90/90 & 75/45 & 40/0 & 0/0 \\
Move pen to box & 95/95 & 60/0 & 80/80 & 90/0 \\
Move seal cup next to box & 90/90 & 100/90 & 80/80 & 10/10 \\
Move seal next to box & 100/100 & 100/100 & 100/100 & 95/95 \\
Move seal next to pencup & 90/90 & 90/90 & 100/100 & 10/10 \\
Move seal onto book & 95/95 & 100/75 & 100/100 & 15/15 \\
Move seal onto table & 100/100 & 100/40 & 100/100 & 90/90 \\
Open drawer & 100/100 & 100/95 & 60/60 & 60/60 \\
Organize table & 20/20 & 5/5 & 0/0 & 0/0 \\
Pick apple from bowl & 100/10 & 90/35 & 40/0 & 15/0 \\
Pick apple from sink & 95/95 & 100/80 & 0/0 & 0/0 \\
Pick bottle from fridge & 100/100 & 95/35 & 100/100 & 100/100 \\
Pick boxdrink from basket & 5/5 & 0/0 & 100/60 & 45/0 \\
Pick can from basket & 90/5 & 100/0 & 100/0 & 95/45 \\
Pick can from cabinet & 100/100 & 100/70 & 60/40 & 80/80 \\
Pick fork from sink & 100/100 & 100/95 & 0/0 & 0/0 \\
Pick hamburger from microwave & 50/50 & 45/15 & 0/0 & 25/25 \\
Pick milk box from fridge & 100/100 & 95/90 & 100/100 & 90/90 \\
Pick sauce can from cabinet & 100/100 & 100/5 & 100/100 & 100/100 \\
Place bowl in dishrack & 85/85 & 100/100 & 100/100 & 30/30 \\
Put book in fileholder & 50/50 & 5/5 & 100/60 & 30/15 \\
Put book on book & 85/85 & 95/95 & 40/20 & 60/60 \\
Put bottle in basket & 100/100 & 50/25 & 80/80 & 90/90 \\
Put bottle in fridge & 90/90 & 30/0 & 100/100 & 95/95 \\
Put bowl in sink & 0/0 & 0/0 & 60/60 & 10/10 \\
Put bread on board & 100/95 & 85/60 & 80/0 & 60/0 \\
Put can close cabinet & 95/90 & 100/20 & 80/80 & 0/0 \\
Put can in cabinet & 85/85 & 100/10 & 80/60 & 95/95 \\
Put can infront of microwave & 85/85 & 100/70 & 100/100 & 95/95 \\
Put can next to basket & 40/25 & 95/95 & 100/100 & 65/65 \\
Put cup in box & 100/100 & 90/90 & 100/100 & 95/95 \\
Put cup on coaster & 70/70 & 0/0 & 100/100 & 80/80 \\
Put cup on table & 80/80 & 65/60 & 80/80 & 45/45 \\
Put glue in box & 100/50 & 95/50 & 100/100 & 80/70 \\
Put hamburger in microwave & 100/95 & 70/55 & 0/0 & 0/0 \\
Put milk box in fridge & 90/90 & 90/40 & 100/100 & 90/90 \\
Put milktea next to laptop & 65/65 & 65/40 & 80/60 & 60/30 \\
Put milktea on shelf & 65/65 & 40/25 & 80/80 & 75/40 \\
Put mouse next to stapler & 90/90 & 95/80 & 100/0 & 70/0 \\
Put mouse on pad & 85/85 & 30/30 & 80/80 & 45/5 \\
Put pen in box & 100/100 & 95/50 & 100/100 & 35/35 \\
Put pen in pencup & 20/15 & 45/20 & 0/0 & 0/0 \\
Put phone next to cube & 90/90 & 100/50 & 60/0 & 45/30 \\
Put phone on holder & 70/70 & 70/65 & 60/60 & 55/55 \\
Put plate in sink & 70/70 & 50/50 & 60/50 & 45/35 \\
Put rubikscube in drawer & 95/95 & 90/85 & 80/80 & 50/45 \\
Put rubikscube next to milktea & 60/50 & 20/20 & 80/20 & 30/0 \\
Put sauce can in basket & 95/75 & 100/65 & 100/100 & 95/55 \\
Put sauce can in cabinet & 100/100 & 90/5 & 100/100 & 85/70 \\
Put seal in box & 100/100 & 100/100 & 100/100 & 100/100 \\
Put spoon in dishrack & 90/90 & 95/95 & 20/20 & 30/30 \\
Put spoon in sink & 100/100 & 100/100 & 0/0 & 35/35 \\
Put spoon on plate & 50/50 & 50/50 & 40/0 & 70/0 \\
Put stapler in drawer & 95/95 & 95/50 & 50/40 & 80/80 \\
Put stapler next to mouse & 90/85 & 20/0 & 0/0 & 10/0 \\
Put stapler on book & 100/95 & 100/100 & 100/100 & 80/80 \\
Set up table & 40/35 & 5/0 & 0/0 & 0/0 \\
Store rubikscube on shelf & 75/75 & 60/40 & 0/0 & 0/0 \\
Store stapler in drawer & 60/60 & 55/35 & 0/0 & 0/0 \\
Switch can with bottle in basket & 75/0 & 35/10 & 60/0 & 0/0 \\
\midrule
Average & 78.8/73.7 & 69.8/49.4 & 60.9/49.9 & 45.3/35.3 \\
\bottomrule
\end{tabular}
\end{table*}
\clearpage
\end{document}